\documentclass[a4paper,12pt]{article}

\usepackage{a4wide}

\usepackage[utf8]{inputenc}
\usepackage[unicode=true,hypertexnames=false]{hyperref}
\usepackage{amsmath,amssymb,amsthm,amsfonts}
\usepackage{dsfont} 
\usepackage{mathtools} 
\usepackage[linesnumbered,ruled]{algorithm2e}
\let\oldnl\nl
\newcommand{\nonl}{\renewcommand{\nl}{\let\nl\oldnl}}
\usepackage{xspace}
\usepackage[dvipsnames,table]{xcolor}
\usepackage{soul} 
\definecolor{lightblue}{rgb}{0.7,0.7,1}
\sethlcolor{lightblue} 
\usepackage{tikz}
\usetikzlibrary{decorations.pathreplacing,calc,tikzmark}
\usepackage{multirow}

\usepackage{pgfplots}
\pgfplotsset{compat=newest}
\pgfplotscreateplotcyclelist{myplotcycle}{%
blue\\
red!80!black\\
violet!80!black\\
orange\\
green!70!black\\
cyan!80!black\\
black\\
very thick,blue,dashed\\ 
}
\pgfplotscreateplotcyclelist{threeconfigs}{%
red!80!black\\
green!70!black\\
black\\
}
\pgfplotscreateplotcyclelist{twoconfigsandea}{%
red!80!black\\
green!70!black\\
very thick,blue,dashed\\ 
}

\allowdisplaybreaks

\newcommand{\oea}{\mbox{$(1 + 1)$~EA}\xspace}
\newcommand{\oplea}{\mbox{$(1+\lambda)$~EA}\xspace}
\newcommand{\mpoea}{\mbox{$(\mu+1)$~EA}\xspace}
\newcommand{\mplea}{\mbox{$(\mu+\lambda)$~EA}\xspace}
\newcommand{\mclea}{\mbox{$(\mu,\lambda)$~EA}\xspace}

\newcommand{\ollga}{${(1 + (\lambda , \lambda))}$~GA\xspace}

\newcommand{\onemax}{\textsc{OneMax}\xspace}
\newcommand{\leadingones}{\textsc{Leading\-Ones}\xspace}
\newcommand{\om}{\textsc{OM}\xspace}
\newcommand{\jump}{\textsc{Jump}\xspace}
\newcommand{\N}{{\mathbb N}}
\newcommand{\R}{{\mathbb R}}
\newcommand{\eps}{\varepsilon}

\DeclareMathOperator{\Bin}{Bin}
\DeclareMathOperator{\Geom}{Geom}
\DeclareMathOperator{\pow}{pow}
\DeclareMathOperator*{\argmax}{arg\,max}
\DeclareMathOperator{\total}{total}

\newtheorem{theorem}{Theorem}
\newtheorem{lemma}[theorem]{Lemma}
\newtheorem{corollary}[theorem]{Corollary}

\let\originalleft\left
\let\originalright\right
\renewcommand{\left}{\mathopen{}\mathclose\bgroup\originalleft}
\renewcommand{\right}{\aftergroup\egroup\originalright}

\begin{document}

\title{Lazy Parameter Tuning and Control: Choosing All Parameters Randomly From a Power-Law Distribution}

\author{Denis Antipov\thanks{Corresponding author} \\
		ITMO University \\
  		St. Petersburg, Russia \\
    \and
    Maxim Buzdalov \\
		ITMO University \\
  		St. Petersburg, Russia \\
		\and
		Benjamin Doerr \\
		Laboratoire d'Informatique (LIX), \\
		CNRS, \'Ecole Polytechnique, \\ 
		Institut Polytechnique de Paris \\
		Palaiseau, France \\
} 

\maketitle
{\sloppy

\begin{abstract}
    Most evolutionary algorithms have multiple parameters and their values drastically affect the performance. Due to the often complicated interplay of the parameters, setting these values right for a particular problem (parameter tuning) is a challenging task. This task becomes even more complicated when the optimal parameter values change significantly during the run of the algorithm since then a dynamic parameter choice (parameter control) is necessary.
    
    In this work, we propose a lazy but effective solution, namely choosing all parameter values (where this makes sense) in each iteration randomly from a suitably scaled power-law distribution. To demonstrate the effectiveness of this approach, we perform runtime analyses of the $(1+(\lambda,\lambda))$ genetic algorithm with all three parameters chosen in this manner. We show that this algorithm on the one hand can imitate simple hill-climbers like the $(1+1)$ EA, giving the same asymptotic runtime on problems like OneMax, LeadingOnes, or Minimum Spanning Tree. On the other hand, this algorithm is also very efficient on jump functions, where the best static parameters are very different from those necessary to optimize simple problems. We prove a performance guarantee that is comparable to the best performance known for static parameters. For the most interesting case that the jump size $k$ is constant, we prove that our performance is asymptotically better than what can be obtained with any static parameter choice. We complement our theoretical results with a rigorous empirical study confirming what the asymptotic runtime results suggest.
\end{abstract}

\section{Introduction}
\label{sec:intro}

Evolutionary algorithms (EAs) are general-purpose randomized search heuristics. They are adapted to the particular problem to be solved by choosing suitable values for their parameters. This flexibility is a great strength on the one hand, but a true challenge for the algorithm designer on the other. Missing the right parameter values can lead to catastrophic performance losses.

Despite being a core topic of both theoretical and experimental research, general advice on how to set the parameters of an EA are still rare. The difficulty stems from the fact that different problems need different parameters, different instances of the same problem may need different parameters, and even during the optimization process on one instance the most profitable parameter values may change over time. 

In an attempt to design a simple one-size-fits-all solution, Doerr, Le, Makhmara, and Nguyen~\cite{DoerrLMN17} proposed to use random parameter values chosen independently in each iteration from a power-law distribution (note that random mutation rates were used before~\cite{DoerrDK18,DoerrDK19}, but with different distributions and for different reasons). Mostly via mathematical means, this was shown to be highly effective for the choice of the mutation rate of the \oea when optimizing the jump benchmark, which has the property that the optimal mutation rate depends strongly on the problem instance. More precisely, for a jump function with representation length $n$ and jump size $2 \le k = o(\sqrt n)$, the standard mutation rate $p=1/n$ gives an expected runtime of $(1+o(1)) e n^k$, where $e \approx 2.718$ is Euler's number. The asymptotically optimal mutation rate $p = k/n$ leads to a runtime of $(1+o(1)) n^k (e/k)^k$. Deviating from the optimal rate by a small constant factor increases the runtime by a factor exponential in~$k$. When using the mutation rate $\alpha/n$, where $\alpha \in [1..n/2]$ is sampled independently in each iteration from a power-law distribution with exponent ${\beta > 1}$, the runtime becomes $\Theta(k^{\beta-0.5} n^k (e/k)^k)$, where the constants hidden by the asymptotic notation are independent of $n$ and $k$. Consequently, apart from the small polynomial factor $\Theta(k^{\beta-0.5})$, this randomized mutation rate gives the performance of the optimal mutation rate and in particular also achieves the super-exponential runtime improvement by a factor of $(e/k)^{\Theta(k)}$ over the standard rate $1/n$.

The idea of choosing parameter values randomly according to a power-law distribution was quickly taken up by other works. In~\cite{FriedrichQW18,QuinzanGWF21},
variants of the heavy-tailed mutation operator were proposed and analyzed on TwoMax, Jump, MaxCut, and several sub-modular problems. In~\cite{WuQT18,DoerrZ21aaai,DoerrQ22ppsn}, power-law mutation in multi-objective optimization was studied. In~\cite{CorusOY21}, the authors compared power-law mutation and artificial immune systems. In~\cite{AntipovBD22}, heavy-tailed mutation was regarded for the~\ollga, however again only for a single parameter and this parameter being the mutation rate. Very recently, the first analysis of a heavy-tailed choice of a parameter of the selection operator was conducted~\cite{DangELQ22}.

While optimizing a single parameter is already non-trivial (and the latest work~\cite{AntipovBD22} showed that the heavy-tailed mutation rate can even give results better than any static mutation rate, that is, it can inherit advantages of dynamic parameter choices), the really difficult problem is finding good values for several parameters of an algorithm. Here the often intricate interplay between the different parameters can be a true challenge (see, e.g.,~\cite{Doerr16} for a theory-based determination of the optimal values of three parameters). 

The only attempt to choose randomly more than one parameter was made in~\cite{AntipovD20ppsn} for the \ollga having a three-dimen\-sio\-nal parameter space spanned by the parameters population size $\lambda$, mutation rate $p$, and crossover bias $c$. For this algorithm, first proposed in~\cite{DoerrDE15}, the product $d = pcn$ 
of mutation rate, crossover bias, and representation length describes the expected distance of an offspring from the parent. It was argued heuristically in~\cite{AntipovD20ppsn} that a reasonable parameter setting should have $p = c$, that is, the same mutation rate and crossover bias. With this reduction of the parameter space to two dimensions, the parameter choice in~\cite{AntipovD20ppsn} was made as follows. Independently (and independently in each iteration), both $\lambda$ and $d$ were chosen from a power-law distribution. Mutation rate and crossover bias were both set to $\sqrt{d/n}$ to ensure $p=c$ and $pcn = d$. When using unbounded power-law distributions with exponents $\beta_\lambda=2+\eps$ and $\beta_d=1+\eps'$ with $\eps, \eps' > 0$ any small constants, this randomized way of setting the parameters gave an expected runtime of $e^{O(k)} (\frac nk)^{(1+\eps)k/2}$ on jump functions with jump size $k \ge 3$. This is very similar (slightly better for $k < \frac 1 \eps$, slightly worse for $k>\frac 1 \eps$) to the runtime of $(\frac nk)^{(k+1)/2} e^{O(k)}$ obtainable with the optimal static parameters. This is a surprisingly good performance for a parameter-less approach, in particular, when compared to the runtime of $\Theta(n^k)$ of many classic evolutionary algorithms. Note that both for the static and dynamic parameters only upper bounds were proven\footnote{A lower bound of $(\frac nk)^{k/2} e^{\Theta(k)}$ fitness evaluations on the runtime of the \ollga with static parameters was shown in~\cite{AntipovDK22}, but this bound was proven for the initialization in the local optimum of $\jump_k$ and it does not include the runtime until the algorithm gets to the local optimum from a random solution.}, hence we cannot make a rigorous conclusion on which algorithm performs better on jump. The proofs of these upper bounds however suggest to us that they are tight.

\emph{Our results:} While the work~\cite{AntipovD20ppsn} showed that in principle it can be profitable to choose more than one parameter randomly from a power-law distribution, it relied on the heuristic assumption that one should take the mutation rate equal to the crossover bias. There is nothing wrong with using such heuristic insight, however, one has to question if an algorithm user (different from the original developers of the \ollga) would have easily found this relation $p=c$. 

In this work, we show that such heuristic preparations are not necessary: One can simply choose all three parameters of the \ollga from (scaled) power-law distributions and obtain a runtime comparable to the ones seen before. More precisely, when using the power-law exponents $2+\eps$ for the distribution of the population size and $1 + \eps'$ for the distributions of the parameters $p$ and $c$ and scaling the distributions for $p$ and $c$ by dividing by $\sqrt n$ (to obtain a constant distance of parent and offspring with constant probability), we obtain the same $e^{O(k)} (\frac nk)^{(1+\eps)k/2}$ runtime guarantee as in~\cite{AntipovD20ppsn}.
From our theoretical results one can see that the exact choice of $\eps'$ does not affect the asymptotical runtime neither on easy functions such as \onemax, nor on hard functions such as $\jump_k$. Hence if an algorithm user would choose all exponents as $2+\eps$, which is a natural choice as it leads to a constant expectation and a super-constant variance as usually desired from a power-law distribution, the resulting runtimes would still be $O(n \log n)$ for \onemax and $e^{O(k)} (\frac nk)^{(1+\eps)k/2}$ for jump functions with gap size~$k$.

With this approach, the only remaining design choice is the scaling of the distributions. It is clear that this cannot be completely avoided simply because of the different scales of the parameters (mutation rates are in $[0,1]$, population sizes are positive integers). However, we argue that here very simple heuristic arguments can be employed. For the population size, being a positive integer, we simply use a power-law distribution on the non-negative integers. For the mutation rate and the crossover bias, we definitely need some scaling as both number have to be in $[0,1]$. Recalling that (and this is visible right from the algorithm definition) the expected distance of offspring from their parents in this algorithm is $d = pcn$ and recalling further the general recommendation that EAs should generate offspring with constant Hamming distance from the parent with reasonable probability (this is, for example, implicit both in the general recommendation to use a mutation rate of $1/n$ and in the heavy-tailed mutation operator proposed in~\cite{DoerrLMN17}), a scaling leading to a constant expected value of $d$ appears to be a good choice. We obtain this by taking both $p$ and $c$ from power-law distributions on the positive integers scaled down by a factor of $\sqrt n$. This appears again to be the most natural choice. We note that if an algorithm user would miss this scaling and scale down both $p$ and $c$ by a factor of $n$ (e.g., to obtain an expected constant number of bits flipped in the mutation step), then our runtime estimates would increase by a factor of $n^{\frac{\beta_p + \beta_c}{2} - 1}$, which is still not much compared to the roughly $n^{k/2}$ runtimes we have and the $\Theta(n^k)$ runtimes of many simple evolutionary algorithms. 

Our precise result is a mathematical runtime analysis of this heavy-tailed algorithm for arbitrary parameters of the three heavy-tailed distributions (power-law exponent and upper bound on the range of positive integers it can take, including the case of no bound) on a set of ``easy'' problems (\onemax, \leadingones, the minimum spanning tree and the partition problem) and on \jump function. We show that on easy problems the heavy-tailed \ollga asymptotically is not worse than the \oea, and on \jump it significantly outperforms the \oea for a wide range of the parameters of power-law distributions.
These results show that the absolutely best performance can be obtained by guessing correctly suitable upper bounds on the ranges. Since guessing these parameters wrong can lead to significant performance losses, whereas the gains from these optimal parameter values are not so high, we would rather advertise our parameter-less ``layman recommendation'' to use unrestricted ranges and power-law exponents slightly more than two for the population size and slightly more than one for other parameters.
These recommendations are supported by the empirical study shown in Section~\ref{sec:experiments}.

Our work also provides an example where a dynamic (here simply randomized) parameter choice provably gives an asymptotic runtime improvement. This improvement is significantly more pronounced than the $o(\sqrt{\log n})$ factor speed-up observed in~\cite{DoerrD18,AntipovBD22} for the optimization of \onemax via the \ollga.

We note that our situation  is different, e.g., from the optimization of jump functions via the \oea. Here the mutation rate $\frac{k}{n}$ is asymptotically optimal~\cite{DoerrLMN17} for $\jump_k$. Clearly, for the easy \onemax-type part of the optimization process, the mutation rate $\frac 1n$ would be superior, but the damage from using the larger rate $\frac kn$ only leads to a lower-order increase of the runtime.

We prove that this is different for the optimization of the jump functions via the \ollga. Since this effect is already visible for constant values of $k$, and in fact strongest visible, to ease the presentation, we assume that $k$ is constant. We note that only for constant $k$ the different variants of the \ollga had a polynomial runtime, so clearly, $k$ constant (and not too large) is the most interesting case.

For constant $k$, our result is $e^{O(k)} (\frac nk)^{(1+\eps)k/2}$. The best runtime that could be obtained with a static mutation rate was $e^{O(k)} n^{(k + 1)/2} k^{-k/2}$. Hence by choosing $\eps$ sufficiently small, our upper bound is asymptotically smaller than the best upper bound for static parameters. Unfortunately, no lower bounds were proven in~\cite{AntipovDK22} for static parameters. To rigorously support our claim that dynamic parameter choices can asymptotically outperform static ones when optimizing jump functions via the \ollga, in Section~\ref{sec:recommendations} we prove such a lower bound. Since this is not the main topic of this work, we shall not go as far as proving that the upper bound for static parameters is tight, but we content ourselves with a simpler proof of a weaker lower bound, which however suffices to support our claim of the superiority of dynamic parameter choices.

In summary, our results demonstrate that choosing all parameters of an algorithm randomly according to a simple (scaled) power-law can be a good way to overcome the problem of choosing appropriate fixed or dynamic parameter values. We are optimistic that this approach will lead to a good performance also for other algorithms and other optimization problems.

\section{Preliminaries}
\label{sec:preliminaries}

In this section we collect definitions and tools which we use in the paper. To avoid misreading of our results, we note that we use the following notation. By $\N$ we denote the set of all positive integer numbers and by $\N_0$ we denote the set of all non-negative integer numbers. We write $[a..b]$ to denote an integer interval including its borders and $[a, b]$ to denote a real-valued interval including its borders. 
For any probability distribution $\mathcal{L}$ and random variable $X$, we write $X\sim\mathcal{L}$ to indicate that $X$ follows the law $\mathcal{L}$.
We denote the binomial law with parameters $n \in \N$ and $p \in [0,1]$ by $\Bin\left(n, p\right)$. We denote the geometric distribution taking values in $\{1, 2, \dots\}$ with success probability $p \in [0,1]$ by $\Geom(p)$. We denote by $T_I$ and $T_F$ the number of iterations and the number of fitness evaluations performed until some event holds (which is always specified in the text). 

\subsection{Objective Functions} 

In this paper we consider five benchmark functions and problems, namely \onemax, \leadingones, the minimum spanning tree problem, the partition problem and $\jump_k$.
All of them are pseudo-Boolean functions, that is, they are defined on the set of bit strings of length $n$ and return a real number.  


\textbf{\onemax} returns the number of one-bits in its argument, that is,
    $\onemax(x) = \om(x) = \sum_{i = 1}^n x_i$.
It is one of the most intensively studied benchmarks in evolutionary computation. Many evolutionary algorithms can find the optimum of \onemax in time $O(n \log n)$~\cite{Muhlenbein92,JansenJW05,Witt06,AntipovD21algo}. The \ollga with a fitness-dependent or self-adjusting choice of the population size~\cite{DoerrDE15,DoerrD18} or with a heavy-tailed random choice of the population size~\cite{AntipovBD21gecco} is capable of solving \onemax  in linear time when the other two parameters are chosen suitably depending on the population size.

\textbf{\leadingones} returns the number of the leading ones in a bit string. In more formal words we maximize function
\begin{align*}
    \leadingones(x) = \sum_{i = 1}^{n} \prod_{j = 1}^i x_j.
\end{align*}
The runtime of the most classic EAs is at least quadratic on \leadingones. More precisely, the runtime of the \oea is $\Theta(n^2)$~\cite{Rudolph97,DrosteJW02}, the runtime of the \mpoea is $\Theta(n^2 + \mu n\log(n))$~\cite{Witt06}, the runtime of the \oplea is $\Theta(n^2 + \lambda n)$~\cite{JansenJW05} and the runtime of the \mplea is $\Omega(n^2 + \frac{\lambda n}{\max\{1,\log(\lambda / n)\}})$~\cite{BadkobehLS14}. It was shown in~\cite{AntipovDK19foga} that the \ollga with standard parameters ($\lambda \in [1..\frac{n}{2}]$, $p = \frac{\lambda}{n}$ and $c = \frac{1}{\lambda}$) also has a $\Theta(n^2)$ runtime on \leadingones.

In the \textbf{minimum spanning tree problem} (MST for brevity) we are given an undirected graph $G = (V, E)$ with positive integer edge weights defined by a weight function $\omega: E \to \N_{\ge 1}$. We assume that this graph does not have parallel edges or loops. The aim is to find a connected subgraph of a minimum weight. By $n$ we denote the number of vertices, by $m$ we denote the number of edges in $G$. 

This problem can be solved by minimizing the following fitness function defined on all subgraphs $G' = (V, E')$ of the given graph $G$.
\begin{align*}
    f(G') = (W_{\total} + 1)^2 cc(G') + (W_{\total}+ 1) |E'| + \sum_{e \in E'} \omega(e),
\end{align*}
where $cc(G')$ is the number of connected components in $G'$ and $W_{\total}$ is the total weight of the graph $G$, that is, the sum of all edge weights. This definition of the fitness guarantees that any connected graph has a better (in this case, smaller) fitness than any unconnected graph and any tree has a better fitness than any graph with cycles.

The natural representation for subgraphs used in~\cite{NeumannW07} is via bit-strings of length $m$, where each bit corresponds to some particular edge in graph $G$. An edge is present in subgraph $G'$ if and only if its corresponding bit is equal to one. In~\cite{NeumannW07} it was shown that the \oea solves the MST problem with the mentioned representation and fitness function in expected number of $O(m^2 \log(W_{\total}))$ iterations. 

In the \textbf{partition problem} we have a set of $n$ objects with positive integer weights $w_1, w_2, \dots, w_n$ and our aim is to split the objects into two sets (usually called \emph{bins}) such that the total weight of the heavier bin is minimal. Without loss of generality we assume that the weights are sorted in a non-increasing order, that is, $w_1 \ge w_2 \ge \dots \ge w_n$. By $w$ we denote the total weight of all objects, that is, $w = \sum_{i = 1}^n w_i$. By a $(1 + \delta)$ approximation (for any $\delta > 0$) we mean a solution in which the weight of the heavier bin is at most by a factor of $(1 + \delta)$ greater than in an optimal solution. 
    
Each partition into two bins can be represented by a bit string of length $n$, where each bit corresponds to some particular object. The object is put into the first bin if and only if the corresponding bit is equal to one. As fitness $f(x)$ of an individual $x$ we consider the total weight of the objects in the heavier bin in the partition which corresponds to $x$.

In~\cite{Witt05} it was shown that the \oea finds a $(\frac{4}{3} + \eps)$ approximation for any constant $\eps > 0$ of any partition problem in linear time and that it finds a $\frac{4}{3}$ approximation in time $O(n^2)$.

The function \textbf{$\jump_k$} (where $k$ is a positive integer parameter) is defined via \onemax as follows.

\begin{align*}
    \jump_k(x) = 
    \begin{cases}
        \om(x) + k, \text{ if } \om(x) \in [0..n - k] \cup \{n\}, \\
        n - \om(x), \text{ if } \om(x) \in [n - k + 1..n - 1].    
    \end{cases}
\end{align*}

A plot of $\jump_k$ is shown in Figure~\ref{fig:jump}. The main feature of $\jump_k$ is a set of local optima at distance $k$ from the global optimum and a valley of extremely low fitness in between. Most EAs optimizing $\jump_k$ first reach the local optima and then have to perform a jump to the global one, which turns out to be a challenging task for most classic algorithms. In particular, for all values of $\mu$ and $\lambda$ it was shown that \mplea and \mclea have a runtime of $\Omega(n^k)$ fitness evaluations when they optimize $\jump_k$~\cite{DrosteJW02,Doerr22}. Using a mutation rate of $\frac kn$~\cite{DoerrLMN17}, choosing it from a power-law distribution~\cite{DoerrLMN17}, or setting it dynamically with a stagnation detection mechanism~\cite{RajabiW20,RajabiW21evocop,RajabiW21gecco,DoerrR22} reduces the runtime of the \oea by a $k^{\Theta(k)}$ factor, however, for constant $k$ the runtime of the \oea remains $\Theta(n^k)$.  Many crossover-based algorithms have a better runtime on $\jump_k$, see~\cite{JansenW02,FriedrichKKNNS16,DangFKKLOSS16,DangFKKLOSS18,RoweA19,WhitleyVHM18} for results on algorithms different from the \ollga. Those beating the $\tilde O(n^{k-1})$ runtime shown in~\cite{DangFKKLOSS18} may appear somewhat artificial and overfitted to the precise definition of the jump function, see~\cite{Witt21}. Outside the world of genetic algorithms, the estimation-of-distribution algorithm cGA and the ant-colony optimizer $2$-MMAS$_{ib}$ were shown to optimize jump functions with small $k=O(\log n)$ in time $O(n \log n)$~\cite{HasenohrlS18,Doerr21cgajump,BenbakiBD21}. Runtime analyses for artificial immune systems, hyperheuristics, and the Metropolis algorithm exist~\cite{CorusOY17,CorusOY19,LissovoiOW19}, but their runtime guarantees are asymptotically weaker than $O(n^k)$ for constant $k$.


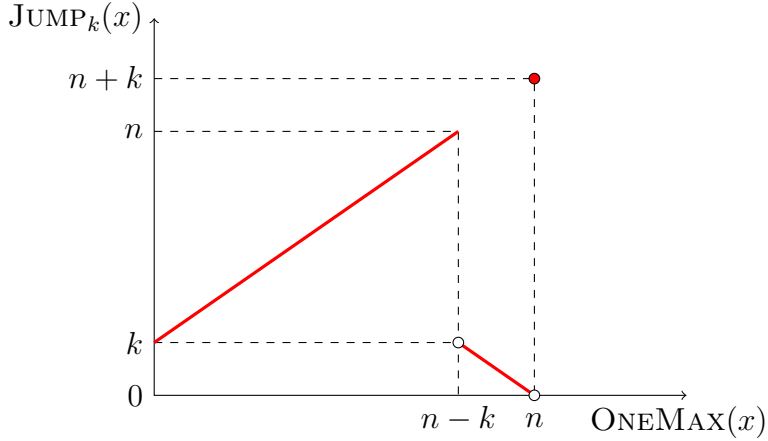
\begin{figure}
    \begin{center}
     \begin{tikzpicture} 
  
      \draw [dashed] (0, 4.2) -- (5, 4.2) -- (5, 0);
      \draw [dashed] (0, 3.5) -- (4, 3.5) -- (4, 0);
      \draw [dashed] (0, 0.7) -- (4, 0.7);
  
      \draw [{<[scale=1.5]}-{>[scale=1.5]}] (0, 5) -- (0, 0) -- (7, 0);
      \draw [very thick, red] (0, 0.7) -- (4, 3.5);
      \draw [very thick, red] (4, 0.7) -- (5, 0);

      \draw [fill=white] (4, 0.7) circle (0.7mm);
      \draw [fill=white] (5, 0) circle (0.7mm);
      \draw [fill=red, draw=black] (5, 4.2) circle (0.7mm);
  
      \node [below] at (6.9, 0) {$\onemax(x)$};
      \node [below] at (5, -0.1) {$n$};
      \node [left] at (0, 0) {$0$};
      \node [below] at (4, 0) {$n - k$};
      \node [left] at (0, 0.7) {$k$};
      \node [left] at (0, 3.5) {$n$};
      \node [left] at (0, 4.2) {$n + k$};
      \node [left] at (0, 5) {$\jump_k(x)$};
  
     \end{tikzpicture}
    \end{center}
    \caption{Plot of the $\jump_k$ function. As a function of unitation, the function value of a search point $x$ depends only on the number of one-bits in~$x$.}
    \label{fig:jump}
  \end{figure}

\subsection{Power-Law Distributions} 

We say that an integer random variable $X$ follows a power-law distribution with parameters $\beta$ and $u$ if 
\begin{align*}
    \Pr[X = i] = 
    \begin{cases}
        C_{\beta, u} i^{-\beta}, \text{ if } i \in [1..u], \\
        0, \text{ else.}
    \end{cases}
\end{align*}
Here $C_{\beta, u} = (\sum_{j = 1}^u j^{-\beta})^{-1}$ denotes the normalization coefficient. We write $X \sim \pow(\beta, u)$ and call $u$ the bounding of $X$ and $\beta$ the power-law exponent. 

The main feature of this distribution is that while having a decent probability to sample $X = \Theta(1)$ (where the asymptotic notation is used for $u \to +\infty$), we also have a good (inverse-polynomial instead of negative-exponential) probability to sample a super-constant value. 
The following lemmas show the well-known properties of the power-law distributions. Their proofs can be found, for example, in~\cite{AntipovD20ppsn}.

\begin{lemma}[Lemma 1 in~\cite{AntipovD20ppsn}]
    \label{lem:sum_estimates}
    For all positive integers $a$ and $b$ such that $b \ge a$ and for all $\beta > 0$, the sum $\sum_{i = a}^b i^{-\beta}$ is
    \begin{itemize}
        \item $\Theta((b + 1)^{1 - \beta} - a^{1 - \beta})$, if $\beta < 1$,
        \item $\Theta(\log(\frac{b + 1}{a}))$, if $\beta = 1$, and
        \item $\Theta(a^{1 - \beta} - (b + 1)^{1 - \beta})$, if $\beta > 1$,
    \end{itemize}
    where $\Theta$ notation is used for $b \to +\infty$.
\end{lemma}

\begin{lemma}[Lemma 2 in~\cite{AntipovD20ppsn}]
    \label{lem:c}
    The normalization coefficient $C_{\beta, u} = (\sum_{j = 1}^u j^{-\beta})^{-1}$ of the power-law distribution with parameters $\beta$ and $u$ is
    \begin{itemize}
        \item $\Theta(u^{\beta - 1})$, if $\beta < 1$,
        \item $\Theta(\frac{1}{\log(u) + 1})$, if $\beta = 1$, and 
        \item $\Theta(1)$, if $\beta > 1$,
    \end{itemize} 
    where $\Theta$ notation is used for $u \to +\infty$.
\end{lemma}

\begin{lemma}[Lemma 3 in~\cite{AntipovD20ppsn}]
    \label{lem:expectation}
    The expected value of the random variable $X\sim \pow(\beta, u)$ is
    \begin{itemize}
        \item $\Theta(u)$, if $\beta < 1$,
        \item $\Theta(\frac{u}{\log(u) + 1})$, if $\beta = 1$,
        \item $\Theta(u^{2 - \beta})$, if $\beta \in (1, 2)$,
        \item $\Theta(\log(u) + 1)$, if $\beta = 2$, and 
        \item $\Theta(1)$, if $\beta > 2$,
    \end{itemize}   
    where $\Theta$ notation is used for $u \to +\infty$.
\end{lemma}

\begin{lemma}
    \label{lem:expectation-square}
    If $X\sim \pow(\beta, u)$, then $E[X^2]$ is
    \begin{itemize}
        \item $\Theta(u^2)$, if $\beta < 1$,
        \item $\Theta(\frac{u^2}{\log(u) + 1})$, if $\beta = 1$,
        \item $\Theta(u^{3 - \beta})$, if $\beta \in (1, 3)$,
        \item $\Theta(\log(u) + 1)$, if $\beta = 3$, and 
        \item $\Theta(1)$, if $\beta > 3$,
    \end{itemize}   
    where $\Theta$ notation is used for $u \to +\infty$.
\end{lemma}

Lemma~\ref{lem:expectation-square} simply follows from Lemma~\ref{lem:sum_estimates}.

\subsection{The Heavy-Tailed \ollga}
\label{sec:ollga}

We now define the heavy-tailed \ollga. The main difference from the standard \ollga is that at the start of each iteration the mutation rate $p$, the crossover bias $c$, and the population size $\lambda$ are randomly chosen as follows. We sample $p \sim n^{-1/2}\pow(\beta_p, u_p)$ and $c \sim n^{-1/2}\pow(\beta_c, u_c)$. 
The population size is chosen via $\lambda \sim \pow(\beta_\lambda, u_\lambda)$. Here the upper limits $u_\lambda$, $u_p$ and $u_c$ can be any positive integers, except we require $u_p$ and $u_c$ to be at most $\sqrt{n}$ (so that we choose both $p$ and $c$ from interval $(0, 1]$). The power-law exponents $\beta_\lambda$, $\beta_p$ and $\beta_c$ can be any non-negative real numbers. We call these parameters of the power-law distribution the \emph{hyperparameters of the heavy-tailed \ollga} and we give recommendations on how to choose these hyperparameters in Section~\ref{sec:recommendations}. The pseudocode of this algorithm is shown in Algorithm~\ref{alg:pseudo}. We note that it is not necessary to store the whole offspring population, since only the best individual has a chance to be selected as a mutation or crossover winner. Hence also large values for $\lambda$ are algorithmically feasible.

Concerning the scalings of the power-law distributions, we find it natural to choose the integer parameter $\lambda$ from a power-law distribution without any normalization. For the scalings of the power-law determining the parameters $p$ and $c$, we argued already in the introduction that the scaling factor of $n^{-1/2}$ is natural as it ensures that the Hamming distance between parent and offspring, which is $pcn$ for this algorithm, is one with constant probability. We see that there is some risk that an algorithm user misses this argument and, for example, chooses a scaling factor of $n^{-1}$ for the mutation rate, which leads to the Hamming distance between parent and mutation offspring being one with constant probability. A completely different alternative would be to choose $c \sim \frac{1}{\pow(\beta_m, u_m)}$, inspired by the recommendation ``$c := 1/(pn)$'' made for static parameters in~\cite{DoerrDE15}. Without proof, we note that these and many similar strategies increase the runtime by at most a factor of $\Theta(n^c)$, $c$ a constant independent of $n$ and $k$, thus not changing the general $n^{(0.5+\eps) k}$  
runtime guarantee proven in this work.

\begin{algorithm}[h]
    $x \gets $ random bit string of length $n$\;
    \While{not terminated}
        {
        Choose $p \sim n^{-1/2}\pow(\beta_p, u_p)$\;
        Choose $c \sim n^{-1/2}\pow(\beta_c, u_c)$\;
        Choose $\lambda \sim \pow(\beta_\lambda, u_\lambda)$\;
        \nonl\textbf{Mutation phase:}\\
        Choose $\ell \sim \Bin\left(n, p\right)$\;
        \For{$i \in [1..\lambda]$}
            {$x^{(i)} \gets$ a copy of $x$\;
            Flip $\ell$ bits in $x^{(i)}$ chosen uniformly at random\;
            }
        $x' \gets \argmax_{z \in \{x^{(1)}, \dots, x^{(\lambda)}\}}f(z)$\;
        \nonl\textbf{Crossover phase:}\\
        \For{$i \in [1..\lambda]$}
            {Create $y^{(i)}$ by taking each bit from $x'$ with probability $c$ and from $x$ with probability $(1 - c)$\;
            }
        $y \gets \argmax_{z \in \{y^{(1)}, \dots, y^{(\lambda)}\} }f(z)$\;
        \If{$f(y) \ge f(x)$}
            {
             $x \gets y$\;   
            }
        }
        \caption{The heavy-tailed \ollga maximizing a pseudo-Boolean function~$f: \{0,1\}^n \to \R$.}
    \label{alg:pseudo}
\end{algorithm}

The following theoretical results exist for the \ollga. With optimal static parameters the algorithm solves \onemax in approximately $O(n\sqrt{\log(n)})$ fitness evaluations~\cite{DoerrDE15}. 
The runtime becomes slightly worse on the random satisfiability instances due to a weaker fitness-distance correlation~\cite{BuzdalovD17}. In~\cite{AntipovDK19foga} it was shown that the runtime of the \ollga on \leadingones is the same as the runtime of the most classic algorithms, that is, $\Theta(n^2)$, which means that it is not slower than most other EAs despite the absence of a strong fitness-distance correlation. The analysis of the \ollga with static parameters on $\jump_k$ in~\cite{AntipovDK22} showed that the \ollga (with  uncommon parameters) can find the optimum in $e^{O(k)}(\frac nk)^{(k + 1)/2}$ fitness evaluations, which is roughly a square root of the $\Theta(n^k)$ runtime of many classic algorithms on this function.

Concerning dynamic parameter choices, a fitness-dependent parameter choice was shown to give linear runtime on \onemax~\cite{DoerrDE15}, which is the best known runtime for  crossover-based algorithms on \onemax. In~\cite{DoerrD18}, it was shown that also the self-adjusting approach of controlling the parameters with a simple one-fifth rule can lead to this linear runtime. The adapted one-fifth rule with a logarithmic cap lets the \ollga outperform the \oea on random satisfiability instances~\cite{BuzdalovD17}. 

Choosing $\lambda$ from a power-law distribution and taking $p = \frac{\lambda}{n}$ and $c = \frac{1}{\lambda}$ lets the \ollga optimize \onemax in linear time~\cite{AntipovBD22}. Also, as it was mentioned in the introduction, with randomly chosen parameters (but with some dependencies between several of them) the \ollga can optimize $\jump_k$ in time of $e^{O(k)}(\frac nk)^{(1 + \eps)k/2}$~\cite{AntipovD20ppsn}. For the \leadingones it was shown in~\cite{AntipovDK19foga} that the runtime of the \ollga is $\Theta(n^2)$ and that any dynamic choice of $\lambda$ does not change this asymptotical runtime.



In our proofs we use the following language (also for the \ollga with static parameters). 
When we analyse the \ollga on \jump and the algorithm has already reached the local optimum, then we call the mutation phase \emph{successful} if all $k$ zero-bits of $x$ are flipped to ones in the mutation winner~$x'$. 
We call the crossover phase \emph{successful} if the crossover winner has a greater fitness than $x$. 

\subsection{Useful Tools}

An important tool in our analysis is Wald's equation~\cite{Wald45} as it allows us to express the expected number of fitness evaluations through the expected number of iterations and the expected cost of one iteration. 

\begin{lemma}[Wald's equation]\label{lem:wald}
	Let $(X_t)_{t \in \N}$ be a sequence of real-valued random variables and let $T$ be a positive integer random variable. Let also all following conditions be true.
	\begin{enumerate}
		\item All $X_t$ have the same finite expectation.
		\item For all $t \in \N$ we have $E[X_t \mathds{1}_{\{T \ge t\}}] = E[X_t] \Pr[T \ge t]$.
		\item $\sum_{t = 1}^{+\infty} E[|X_t| \mathds{1}_{\{T \ge t\}}] < \infty$.
		\item $E[T]$ is finite.
	\end{enumerate}
	Then we have
	\[
		E\left[\sum_{t = 1}^{T} X_t\right] = E[T]E[X_1].	
	\]
\end{lemma}


In our analysis of the heavy-tailed \ollga we use the following multiplicative drift theorem.

\begin{theorem}[Multiplicative Drift~\cite{DoerrJW12algo}]
    \label{thm:mult-drift}
    Let $S \subset \R$ be a finite set of positive numbers with minimum $s_{\min}$. Let $\{X_t\}_{t \in \N_0}$ be a sequence of random variables over $S \cup \{0\}$. Let $T$ be the first point in time $t$ when $X_t = 0$, that is,
    \[
      T = \min\{t \in \N: X_t = 0\},  
    \] 
    which is a random variable. Suppose that there exists a constant $\delta > 0$ such that for all $t \in \N_0$ and all $s \in S$ such that $\Pr[X_t = s] > 0$ we have
    \[
        E[X_t - X_{t + 1} \mid X_t = s] \ge \delta s.        
    \]
    Then for all $s_0 \in S$ such that $\Pr[X_0 = s_0] > 0$ we have
    \[
      E[T \mid X_0 = s_0] \le \frac{1 + \ln(s_0/s_{\min})}{\delta}.
    \]
\end{theorem}

We use the following well-known relation between the arithmetic and geometric means.

\begin{lemma}\label{lem:a-g-means}
    For all positive $a$ and $b$ it holds that $a + b \ge 2\sqrt{ab}$.
\end{lemma}

\section{Runtime Analysis}
\label{sec:theory}

In this section we perform a runtime analysis of the heavy-tailed \ollga on the easy problems \onemax, \leadingones, and Minimum Spanning Tree as well as the more difficult \jump problem. We show that this algorithm can efficiently escape local optima and that it is capable of solving \jump functions much faster than the known mutation-based algorithms and most of the crossover-based EAs. At the same time it does not fail on easy functions like \onemax, unlike the \ollga with those static parameters which are optimal for \jump~\cite{AntipovDK22}.

From the results of this section we distill the recommendations to use $\beta_p$ and $\beta_c$ slightly greater than one and to use $\beta_\lambda$ slightly greater than two. We also suggest to use almost unbounded power-law distributions, taking $u_c = u_p = \sqrt{n}$ and $u_\lambda = 2^n$. These recommendations are justified in Corollary~\ref{cor:hyperparameters}.

\subsection{Easy Problems}
\label{sec:onemax}

In this subsection we show that the heavy-tailed \ollga has a reasonably good performance on the easy problems \onemax, \leadingones, minimum spanning tree, and partition.

\subsubsection{\onemax}

The following result shows that the heavy-tailed \ollga just like the simple \oea solves the \onemax problem in $O(n\log(n))$ iterations.

\begin{theorem}\label{thm:onemax}
  If $\beta_\lambda > 1$, $\beta_p > 1$, and $\beta_c > 1$, then the heavy-tailed \ollga finds the optimum of \onemax in $O(n\log(n))$ iterations. The expected number of fitness evaluations is
  \begin{itemize}
      \item $O(n\log(n))$, if $\beta_\lambda > 2$,
      \item $O(n\log(n)(\log(u_\lambda) + 1))$, if $\beta_\lambda = 2$, and
      \item $O(nu_\lambda^{2 - \beta_\lambda} \log(n))$, if $\beta_\lambda \in (1, 2)$.
  \end{itemize}
\end{theorem}

The central argument in the proof of Theorem~\ref{thm:onemax} is the observation that the heavy-tailed \ollga performs an iteration equivalent to one of the \oea with a constant probability, which is shown in the following lemma.

\begin{lemma}\label{lem:opo_iteration}
    If $\beta_p$, $\beta_c$ and $\beta_\lambda$ are all strictly greater than one, then with probability $\rho = \Theta(1)$ the heavy-tailed \ollga chooses $p = c = \frac{1}{\sqrt{n}}$ and $\lambda = 1$ and performs an iteration of the \oea with mutation rate $\frac{1}{n}$.
\end{lemma}

\begin{proof}
    Since we choose $p$, $c$ and $\lambda$ independently, then by the definition of the power-law distribution and by Lemma~\ref{lem:c} we have
    \begin{align*}
        \rho &= \Pr\left[p = \frac{1}{\sqrt{n}}\right] \Pr\left[c = \frac{1}{\sqrt{n}}\right] \Pr\left[\lambda = 1\right] \\
             &= C_{\beta_p,u_p} 1^{-\beta_p} \cdot C_{\beta_c,u_c} 1^{-\beta_c} \cdot C_{\beta_\lambda,u_\lambda} 1^{-\beta_\lambda} \\
             &= \Theta(1) \cdot \Theta(1) \cdot \Theta(1) = \Theta(1).
    \end{align*}

    If we have $\lambda = 1$, then we have only one mutation offspring which is automatically chosen as the mutation winner $x'$. Note that although we first choose $\ell \sim \Bin(n, p)$ and then flip $\ell$ random bits in $x$, the distribution of $x'$ in the search space is the same as if we flipped each bit independently with probability $p$ (see Section~2.1 in~\cite{DoerrDE15} for more details).
    
    In the crossover phase we create only one offspring $y$ by applying the biased crossover to $x$ and $x'$. Each bit of this offspring is different from the bit in the same position in $x$ if and only if it was flipped in $x'$ (with probability $p$) and then taken from $x'$ in the crossover phase (with probability $c$). Therefore, $y$ is distributed in the search space as if we generated it by applying the standard bit mutation with mutation rate $pc$ to $x$. Hence, we can consider such iteration of the heavy-tailed \ollga as an iteration of the \oea which uses a standard bit mutation with mutation rate $pc = \frac{1}{n}$.
\end{proof}

We are now in position to prove Theorem~\ref{thm:onemax}.

\begin{proof}[Proof of Theorem~\ref{thm:onemax}]
    By Lemma~\ref{lem:opo_iteration} with probability at least $\rho$, which is at least some constant independent of the problem size $n$, the heavy-tailed \ollga performs an iteration of the \oea. Hence, the probability $P(i)$ to increase fitness in one iteration if we have already reached fitness $i$ is
    \[
        P(i) \ge \rho \cdot \frac{n - i}{en}.
    \] 
    Hence, we estimate the total runtime in terms of iterations as a sum of the expected runtimes until we leave each fitness level.
    \[
        E[T_I] \le \sum_{i = 0}^{n - 1} \frac{1}{P(i)} \le \frac{1}{\rho} \sum_{i = 0}^{n - 1} \frac{ne}{n - i} \le \frac{en\ln(n)}{\rho} = O(n\log(n)).
    \]
    To compute the expected number of fitness evaluations until we find the optimum we use Wald's equation (Lemma~\ref{lem:wald}). Since in each iteration of the heavy-tailed \ollga we make $2\lambda$ fitness evaluations, we have
    \[
      E[T_F] = E[T_I] \cdot E[2\lambda]. 
    \]
    By Lemma~\ref{lem:expectation} we have 
    \begin{align*}
        E[\lambda] = \begin{cases}
                \Theta(1), &\text{ if } \beta_\lambda > 2, \\
                \Theta(\log(u_\lambda) + 1), &\text{ if } \beta_\lambda = 2, \\
                \Theta(u_\lambda^{2 - \beta_\lambda}), &\text{ if } \beta_\lambda \in (1, 2). \\
        \end{cases}
    \end{align*}
    Therefore,
    \begin{align*}
        E[T_F] = \begin{cases}
            O(n\log(n)), &\text{ if } \beta_\lambda > 2, \\
            O(n\log(n)(\log(u_\lambda) + 1)), &\text{ if } \beta_\lambda = 2, \\
            O(n u_\lambda^{2 - \beta_\lambda} \log(n)), &\text{ if } \beta_\lambda \in (1, 2). \\
        \end{cases}
    \end{align*}
\end{proof}

Theorem~\ref{thm:onemax} shows that the heavy-tailed \ollga can fall back to a \oea behavior and turn into a simple hill climber. Since we do not have a matching lower bound, our analysis leaves open the question to what extent the heavy-tailed \ollga benefits from iterations in which it samples parameter values different from the ones used in the lemma above. On the one hand, in~\cite{AntipovBD22} it was shown that if we choose only one parameter $\lambda$ from the power-law distribution and set the other parameters to their optimal values in the \ollga (namely, $p = \frac{\lambda}{n}$ and $c = \frac{1}{\lambda}$~\cite{DoerrDE15}), then we have a linear runtime on \onemax. This indicates that there is a chance that the heavy-tailed \ollga with an independent choice of three parameters can also have a $o(n\log(n))$ runtime on this problem. On the other hand, the probability that we choose $p$ and $c$ close to their optimal values is not high, hence we have to rely on making good progress when using non-optimal parameters values. Our experiments presented in Section~\ref{sec:experiments-onemax} suggest that such parameters do not yield the desired progress speed and that the heavy-tailed \ollga has an $\Omega(n\log(n))$ runtime (see Figure~\ref{plot:om:pc}). For this reason, we rather believe that the heavy-tailed \ollga proposed in this work has an inferior performance on \onemax than the one proposed in~\cite{AntipovBD22}. Since our new algorithm has a massively better performance on jump functions, we feel that losing a logarithmic factor in the runtime on \onemax is not too critical.

Lemma~\ref{lem:opo_iteration} also allows us to transform any upper bound on the runtime of the \oea which was obtained via the fitness level argument or via drift with the fitness into the same asymptotical runtime for the heavy-tailed \ollga. We give three examples in the following subsections.

\subsubsection{\leadingones}

For the \leadingones problem, we now show that arguments analogous to the ones in~\cite{Rudolph97} can be used to prove an $O(n^2)$ runtime guarantee also for the heavy-tailed \ollga. 

\begin{theorem}
    If $\beta_\lambda > 1$, $\beta_p > 1$, and $\beta_c > 1$, then the expected runtime of the heavy-tailed \ollga on \leadingones is $O(n^2)$ iterations. In terms of fitness evaluations the expected runtime is
    \begin{align*}
        E[T_F] = \begin{cases}
            O(n^2), &\text{ if } \beta_\lambda > 2, \\
            O(n^2(\log(u_\lambda) + 1)), &\text{ if } \beta_\lambda = 2, \\
            O(n^2 u_\lambda^{2 - \beta_\lambda}), &\text{ if } \beta_\lambda \in (1, 2). \\
        \end{cases}
    \end{align*}
\end{theorem}

\begin{proof}
    The probability that the heavy-tailed \ollga improves the fitness in one iteration is at least the probability that it performs an iteration of the \oea that improves the fitness. By Lemma~\ref{lem:opo_iteration} the probability that the heavy-tailed \ollga performs an iteration of the \oea is $\Theta(1)$. The probability that the \oea increases the fitness in one iteration is at least the probability that it flips the first zero-bit in the string and does not flip any other bit, which is $\frac{1}{n}(1 - \frac{1}{n})^{n - 1} \ge \frac{1}{en}$. Hence, the probability that the heavy-tailed \ollga increases the fitness in one iteration is $\Omega(\frac{1}{n})$.

    Therefore, the expected number of iterations before the heavy-tailed \ollga improves the fitness is $O(n)$ iterations. Since there will be no more than $n$ improvements in fitness before we reach the optimum, the expected total runtime of the heavy-tailed \ollga on \leadingones is at most $O(n^2)$ iterations. Since by Lemma~\ref{lem:expectation} with $\beta_\lambda > 1$ the expected cost of one iteration is 
    \begin{align*}
        E[2\lambda] = \begin{cases}
                \Theta(1), &\text{ if } \beta_\lambda > 2, \\
                \Theta(\log(u_\lambda) + 1), &\text{ if } \beta_\lambda = 2, \\
                \Theta(u_\lambda^{2 - \beta_\lambda}), &\text{ if } \beta_\lambda \in (1, 2), \\
        \end{cases}
    \end{align*}
    by Wald's equation (Lemma~\ref{lem:wald}) the expected total runtime in terms of fitness evaluations is
    \begin{align*}
        E[T_F] = E[2\lambda] E[T_I] =\begin{cases}
            O(n^2), &\text{ if } \beta_\lambda > 2, \\
            O(n^2(\log(u_\lambda) + 1)), &\text{ if } \beta_\lambda = 2, \\
            O(n^2 u_\lambda^{2 - \beta_\lambda}), &\text{ if } \beta_\lambda \in (1, 2). \\
        \end{cases}
    \end{align*}
\end{proof}

\subsubsection{Minimum Spanning Tree Problem}

We proceed with the runtime on the minimum spanning tree problem. Reusing some of the arguments from~\cite{NeumannW07} and some more from the later work~\cite{DoerrJW12algo}, we show that the expected runtime of the heavy-tailed \ollga admits the same upper bound $O(m^2\log(W_{\total}))$ as the \oea.

\begin{theorem}
    If $\beta_\lambda > 1$, $\beta_p > 1$, and $\beta_c > 1$, then the expected runtime of the heavy-tailed \ollga on minimum spanning tree problem is $O(m^2 \log(W_{\total}))$ iterations. In terms of fitness evaluations it is
    \begin{align*}
        E[T_F] = \begin{cases}
            O(m^2\log(W_{\total})), &\text{ if } \beta_\lambda > 2, \\
            O(m^2\log(W_{\total})(\log(u_\lambda) + 1)), &\text{ if } \beta_\lambda = 2, \\
            O(m^2 u_\lambda^{2 - \beta_\lambda}\log(W_{\total})), &\text{ if } \beta_\lambda \in (1, 2). \\
        \end{cases}
    \end{align*}
\end{theorem}

\begin{proof}
    In~\cite{NeumannW07} it was shown that starting with a random subgraph of $G$, the \oea finds a spanning tree graph in $O(m\log(n))$ iterations. We now briefly adjust these arguments to the heavy-tailed \ollga. If $G'$ is disconnected, then the probability to reduce the number of connected components is at most the probability that the heavy-tailed \ollga performs an iteration of the \oea multiplied by the probability that an iteration of the \oea adds an edge which connects two connected components (and does not add or remove other edges from the subgraph $G'$). The latter probability is at least $\frac{cc(G') - 1}{m} (1 - \frac{1}{m})^{m - 1} \ge \frac{cc(G') - 1}{em}$, since there are at least $cc(G') - 1$ edges which we can add to connect a pair of connected components. Therefore, by the fitness level argument we have that the expected number of iterations before the heavy-tailed \ollga finds a connected graph is $O(m \log(n))$.

    If the algorithm has found a connected graph, with probability $\Omega(\frac{|E'| - (n - 1)}{m})$ the heavy-tailed \ollga performs an iteration of the \oea that removes an edge participating in a cycle (since there are at least $(|E'| - (n - 1))$ such edges). Therefore, in $O(m\log(m))$ iterations the heavy-tailed \ollga finds a spanning tree (probably not the minimum one). Note that $O(m\log(m)) = O(m\log(n))$, since we do not have loops and parallel edges and thus $m \le \frac{n(n - 1)}{2}$.

    Once the heavy-tailed \ollga has obtained a spanning tree, it cannot accept any subgraph that is not a spanning tree. Therefore, we can use the multiplicative drift argument from~\cite{DoerrJW12algo}. Namely, we define a potential function $\Phi(G')$ that is equal to the weight of the current tree minus the weight of the minimum spanning tree. In~\cite{DoerrJW12algo} it was shown that for every iteration $t$ of \oea, we have
    \begin{align*}
        E[\Phi(G'_t) - \Phi(G'_{t + 1}) \mid \Phi(G'_t) = W] \ge \frac{W}{em^2},
    \end{align*}
    where $G'_t$ denotes the current graph in the start of iteration $t$. By Lemma~\ref{lem:opo_iteration} and since the weight of the current graph cannot decrease in one iteration, for the heavy-tailed \ollga we have
    \begin{align*}
        E[\Phi(G'_t) - \Phi(G'_{t + 1}) \mid \Phi(G'_t) = W] \ge \frac{\rho W}{em^2}
    \end{align*}
    for some $\rho$, which is a constant independent of $m$ and $W$. Since the edge weights are integers, we have $\Phi(G'_t) \ge 1 \eqqcolon s_{\min}$ for all $t$ such that $G'_t$ is not an optimal solution. We also have $\Phi(G'_0) \le W_{\total}$ by the definition of $W_{\total}$. Therefore, by the multiplicative drift theorem (Theorem~\ref{thm:mult-drift}) we have that the expected runtime until we find the optimum starting from a spanning tree is at most
    \begin{align*}
        \frac{1 + \ln(W_{\total})}{\rho / (e m^2)} = O(m^2\log(W_{\total})).
    \end{align*}
    
    Together with the runtime to find a spanning tree, we obtain a total expected runtime of
    \begin{align*}
        E[T_I] = O(m\log(n)) + O(m\log(n)) + O(m^2\log(W_{\total})) = O(m^2\log(W_{\total}))
    \end{align*}
    iterations. By Lemma~\ref{lem:expectation} and by Wald's equation (Lemma~\ref{lem:wald}) the expected number of fitness evaluations is therefore 
    \begin{align*}
        E[T_F] = E[2\lambda] E[T_I] = \begin{cases}
            O(m^2\log(W_{\total})), &\text{ if } \beta_\lambda > 2, \\
            O(m^2\log(W_{\total})(\log(u_\lambda) + 1)), &\text{ if } \beta_\lambda = 2, \\
            O(m^2 u_\lambda^{2 - \beta_\lambda}\log(W_{\total})), &\text{ if } \beta_\lambda \in (1, 2). \\
        \end{cases}
    \end{align*}
\end{proof}

\subsubsection{Approximations for the Partition Problem}

We finally regard the partition problem. We use similar arguments as in~\cite{Witt05} (slightly modified to exploit multiplicative drift analysis) to show that the heavy-tailed \ollga also finds a $(\frac{4}{3} + \eps)$ approximation in linear time. For $\frac{4}{3}$ approximations we improve the $O(n^2)$ runtime result of~\cite{Witt05} and show that both the \oea and the heavy-tailed \ollga succeed in $O(n\log(w))$ fitness evaluations.

\begin{theorem}
    If $\beta_\lambda > 2$, $\beta_p > 1$, and $\beta_c > 1$, then the heavy-tailed \ollga finds a $(\frac{4}{3} + \eps)$ approximation to the partition problem in an expected number of $O(n)$ iterations. The expected number of fitness evaluations is
    \begin{align*}
        E[T_F] = \begin{cases}
            O(n), &\text{ if } \beta_\lambda > 2, \\
            O(n(\log(u_\lambda) + 1)), &\text{ if } \beta_\lambda = 2, \\
            O(n u_\lambda^{2 - \beta_\lambda}), &\text{ if } \beta_\lambda \in (1, 2). \\
        \end{cases}
    \end{align*}
    The heavy-tailed \ollga and the \oea also find a $\frac{4}{3}$ approximation in an expected number of $O(n\log(w))$ iterations. The expected number of fitness evaluations for the heavy-tailed \ollga is
    \begin{align*}
        E[T_F] = \begin{cases}
            O(n\log(w)), &\text{ if } \beta_\lambda > 2, \\
            O(n\log(w)(\log(u_\lambda) + 1)), &\text{ if } \beta_\lambda = 2, \\
            O(n u_\lambda^{2 - \beta_\lambda}\log(w)), &\text{ if } \beta_\lambda \in (1, 2). \\
        \end{cases}
    \end{align*}
\end{theorem}

\begin{proof}

    We first recall the definition of a \emph{critical object} from~\cite{Witt05}. Let $\ell \ge \frac{w}{2}$ be the fitness of the optimal solution. Let $i_1 < i_2 < \dots < i_k$ be the indices of the objects in the heavier bin. Then we call the object $r$ in the heavier bin the critical one if it is the object with the smallest index such that
    \begin{align*}
        \sum_{j: i_j \le r} w_{i_j} > \ell.
    \end{align*}
    In other words, the critical object is the object in the heavier bin such that the total weight of all previous (non-lighter) objects in that bin is not greater than $\ell$, but the total weight of all previous objects together with the weight of this object is greater than~$\ell$. We call the weight of the critical object the \emph{critical weight}. We also call the objects in the heavier bin which have index at least $r$ the \emph{light objects}. This notation is illustrated in Figure~\ref{fig:critical-object}.
    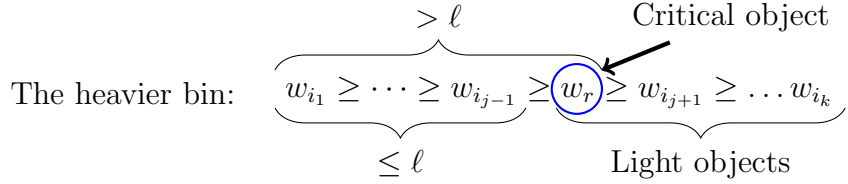
\begin{figure}
        \begin{center}
            \begin{tikzpicture}
                \node (bin) at (0, 0) {The heavier bin:};
                \node [right = 2ex] at (bin.east) {$w_{i_1} \ge \dots \ge w_{i_{j - 1}} \ge w_r \ge w_{i_{j + 1}} \ge \dots w_{i_k}$};

                \draw [decorate,decoration={brace, amplitude=0.4cm}] (5.3, -0.2) -- (2, -0.2) node [pos=0.5,below=1em] {$\le \ell$};
                \draw [decorate,decoration={brace, amplitude=0.4cm}] (2, 0.3) -- (6.3, 0.3) node [pos=0.5,above=1em] {$> \ell$};

                \draw [decorate,decoration={brace, amplitude=0.4cm}] (9.5, -0.2) -- (5.7, -0.2) node [pos=0.5,below=1em] {Light objects};

                \draw [blue, thick] (5.97, 0.05) circle (0.33cm);

                \node (crobj) at (8, 1) {Critical object};
                \draw [->, ultra thick] (crobj) -- (6.3, 0.3);

            \end{tikzpicture}
        \end{center}
        \caption{Illustration of the definition of the critical object.}
        \label{fig:critical-object}
    \end{figure}

    We now show that at some moment the critical weight becomes at most $\frac{w}{3}$ and does not exceed this value in the future. For this we consider two cases.
    
    \textbf{Case 1: $w_2 > \frac{w}{3}$.} Note that in this case we also have $w_1 > \frac{w}{3}$, since $w_1 \ge w_2$ and the weight of all other objects is $w - w_1 - w_2 < \frac{w}{3}$. If the two heaviest objects are in the same bin, then the weight of this (heavier) bin is at least $\frac{2w}{3}$. In any partition in which these two objects are separated the weight of the heavier bin is at most $\max\{w - w_1, w - w_2\} < \frac{2w}{3}$, therefore if the algorithm generates such a partition it would replace a partition in which the two heaviest objects are in the same bin. For the same reason, once we have a partition with the two heaviest objects in different bins, we cannot accept a partition in which they are in the same bin.

    The probability of separating the two heaviest objects into two different bins is at least the probability that the heavy-tailed \ollga performs an iteration of the \oea (which by Lemma~\ref{lem:opo_iteration} is $\Theta(1)$) multiplied by the probability that in this iteration we move one of these two objects into a different bin and do not move the second object. This is at least
    \begin{align*}
        \Theta(1) \cdot \frac{2}{n} \left(1 - \frac{1}{n}\right) = \Theta\left(\frac{1}{n}\right).
    \end{align*}  
    Consequently, in an expected number of $O(n)$ iterations the two heaviest objects will be separated into different bins. 
    
    Note that the weight of the heaviest object cannot be greater than the weight of the heavier bin (even in the optimal solution), hence we have $w_2 \le w_1 \le \ell$. Therefore, when the two heaviest objects are separated into different bins neither of them can be the critical one. Hence, the critical weight is now at most $w_3 < \frac{w}{3}$.  
    
    \textbf{Case 2: $w_2 \le \frac{w}{3}$.} Since the heaviest object can never be the critical one, the critical weight is at most $w_2 \le \frac{w}{3}$.
    
    Once the critical weight is at most $\frac{w}{3}$, we define a potential function 
    \begin{align*}
        \Phi(x_t) = \max\{f(x_t) - \ell - \tfrac{w}{6}, 0\},
    \end{align*}
    where $x_t$ is the current individual of the heavy-tailed \ollga at the beginning of iteration $t$. Note that this potential function does not increase due to the elitist selection of the \ollga.

    We now show that as long as $\Phi(x_t) > 0$, any iteration which moves any light object to the lighter bin and does not move other objects reduces the fitness (and the potential). Recall that the weight of each light object is at most $\frac{w}{3}$. Then the weight of the bin which was heavier before the move is reduced by the weight of the moved object. The weight of the other bin becomes at most 
    \begin{align*}
        w - f(x_t) + \frac{w}{3} < \ell - \frac{w}{6} + \frac{w}{3} \le \ell + \frac{w}{6}.
    \end{align*}
    Therefore, the weight of both bins becomes smaller than the weight of the bin which was heavier before the move, hence such a partition is accepted by the algorithm.

    Now we estimate the expected decrease of the potential in one iteration. Recall that by Lemma~\ref{lem:opo_iteration} the probability that the heavy-tailed \ollga performs an iteration of the \oea is at least some $\rho = \Theta(1)$. The probability that in such an iteration we move only one particular object is $\frac{1}{n}(1 - \frac{1}{n})^{n - 1} \ge \frac{1}{en}$. Hence we have two options.

    \begin{itemize}
        \item If there is at least one light object with weight at least $\Phi(x_t)$, then moving it we decrease the potential to zero, since the wight of the heavier bin becomes not greater than $\ell + \frac{w}{6}$ and the weight of the lighter bin also cannot become greater than $\ell + \frac{w}{6}$ as it was shown earlier. Hence, we have
        \begin{align*}
            E\left[\Phi(x_t) - \Phi(x_{t + 1}) \mid \Phi(x_t) = s\right] \ge \frac{s\rho}{en}.
        \end{align*}
        \item Otherwise, the move of any light object decreases the potential by the weight of the moved object, since the heavy bin will remain the heavier one after such a move. The total weight of the light objects is at least $f(x_t) - \ell \ge \Phi(x_t)$. Let $L$ be the set of indices of the light objects. Then we have
        \begin{align*}
            E\left[\Phi(x_t) - \Phi(x_{t + 1}) \mid \Phi(x_t) = s\right] \ge \rho \sum_{i \in L} \frac{w_i}{en} \ge \frac{s\rho}{en}.
        \end{align*}
    \end{itemize}

    Now we are in position to use the multiplicative drift theorem (Theorem~\ref{thm:mult-drift}). Note that the maximum value of potential function is $\frac{w}{2}$ and its minimum positive value is $\frac{1}{6}$ (since $f(x_t)$ and $\ell$ are integer values and $\frac{w}{6}$ is divided by $\frac{1}{6}$). Therefore, denoting $T_I$ as the smallest $t$ such that $\Phi(x_t) = 0$, we have
    \begin{align*}
        E[T_I] \le \frac{1 + \ln\left(3w\right)}{\rho/(en)} = \Theta(n\log(w)).
    \end{align*}
    When $\Phi(x_t) = 0$, we have
    \begin{align*}
        f(x_t) \le \ell + \tfrac{w}{6} \le \ell + \tfrac{\ell}{3} = \tfrac{4}{3} \ell,
    \end{align*} 
    which means that $x_t$ is a $\frac{4}{3}$ approximation of the optimal solution.

    To show that we obtain a $(\frac{4}{3} + \eps)$ approximation in expected linear time for all constants $\eps > 0$, we use a modified potential function $\Phi_\eps$, which is defined by
    \begin{align*}
        \Phi_\eps(x_t) = \begin{cases}
            0, & \text{ if } \Phi(x_t) \le \frac{\eps w}{2}, \\
            \Phi(x_t), & \text{ otherwise.}
        \end{cases}
    \end{align*}
    For this potential function the drift is at least as large as for $\Phi$, but its smallest non-zero value is $\frac{\eps w}{2}$. Hence, by the multiplicative drift theorem (Theorem~\ref{thm:mult-drift}) the expectation of the first time $T_I(\eps)$ when $\Phi_\eps$ turns to zero is at most
    \begin{align*}
        E[T_I(\eps)] \le \frac{1 + \ln\left(\frac{w}{6} / \frac{\eps w}{2}\right)}{\rho/(en)} = O(n).
    \end{align*} 
    When $\Phi(x_t) = 0$, we have
    \begin{align*}
        f(x_t) \le \ell + \frac{w}{6} + \frac{\eps w}{2} \le \ell + \frac{\ell}{3} + \eps \ell = \left(\frac{4}{3} + \eps\right) \ell,
    \end{align*} 
    therefore $x_t$ is a $(\frac{4}{3} + \eps)$ approximation.

    By Lemma~\ref{lem:expectation} and by Wald's equation (Lemma~\ref{lem:wald}) we also have the following estimates on the runtimes $T_F$ and $T_F(\eps)$ in terms of fitness evaluations.
    \begin{align*}
        E[T_F] &= E[2\lambda] \cdot E[T_I] = \begin{cases}
            O(n\log(w)), &\text{ if } \beta_\lambda > 2, \\
            O(n\log(w)(\log(u_\lambda) + 1)), &\text{ if } \beta_\lambda = 2, \\
            O(n u_\lambda^{2 - \beta_\lambda}\log(w)), &\text{ if } \beta_\lambda \in (1, 2), \\
        \end{cases} \\
        E[T_F(\eps)] &= E[2\lambda] \cdot E[T_I(\eps)] = \begin{cases}
            O(n), &\text{ if } \beta_\lambda > 2, \\
            O(n(\log(u_\lambda) + 1)), &\text{ if } \beta_\lambda = 2, \\
            O(n u_\lambda^{2 - \beta_\lambda}), &\text{ if } \beta_\lambda \in (1, 2). \\
        \end{cases} \\
    \end{align*}

\end{proof}


\subsection{\jump Functions}
\label{sec:jump}

In this subsection we show that the heavy-tailed \ollga performs well on jump functions, hence there is no need for the informal argumentation~\cite{AntipovD20ppsn} to choose mutation rate $p$ and crossover bias $c$ identical. The main result is the following theorem, which estimates the expected runtime until we leave the local optimum of $\jump_k$.

\begin{theorem}
    \label{thm:jump}
    Let $k \in [2..\frac n4]$, $u_p \ge \sqrt{2k}$, and $u_c \ge \sqrt{2k}$. Assume that we use the heavy-tailed \ollga (Algorithm~\ref{alg:pseudo}) to optimize $\jump_k$, starting already in the local optimum. Then the expected number of fitness evaluations until the optimum is found is shown in Table~\ref{tbl:runtime}, where $p_{pc}$ denotes the probability that both $p$ and $c$ are in $[\sqrt{\frac kn}, \sqrt{\frac {2k}n}]$. Table~\ref{tbl:ppc} shows estimates for $p_{pc}$.
\end{theorem}

\begin{table}[t]
    \caption{Influence of the hyperparameters $\beta_\lambda, u_\lambda$ on the expected number $E[T_F]$ of fitness evaluations the heavy-tailed \ollga starting in the local optimum takes to optimize $\jump_k$. Since all runtime bounds are of type $E[T_F] = F(\beta_\lambda,u_\lambda) / p_{pc}$, where $p_{pc} = \Pr[p \in [\sqrt{\frac kn}, \sqrt{\frac {2k}n}] \wedge c \in [\sqrt{\frac kn}, \sqrt{\frac {2k}n}]]$ and ${F(\beta_\lambda,u_\lambda)}$ is some function of $\beta_\lambda$ and $u_\lambda$, to ease reading we only state $F(\beta_\lambda,u_\lambda) = E[T_F] p_{pc}$ and show the influence of the hyperparameters on $p_{pc}$ in Table~\ref{tbl:ppc}. Asymptotical notation is used for $n \to +\infty$. The highlighted cell shows the result for the hyperparameters suggested in Corollary~\ref{cor:hyperparameters}.}
	\label{tbl:runtime}
	\begin{center}
		\begin{tabular}{|c||c|c|}
			\hline
            & \multicolumn{2}{c|}{$E[T_F]p_{pc}$} \\ \cline{2-3}

             & 
            if $u_\lambda < \left(\frac{n}{k}\right)^{k/2}$ &  
            if $u_\lambda \ge \left(\frac{n}{k}\right)^{k/2}$ \\ \hline

            $\beta_\lambda \in [0, 1)$ & 
            \multirow{3}{*}{$e^{O(k)} \frac{1}{u_\lambda}\left(\frac{n}{k}\right)^k$}  & 
            $u_\lambda e^{O(k)}$ \\ \cline{1-1}\cline{3-3}
            
            $\beta_\lambda = 1$ &  & 
            $\frac{u_\lambda e^{O(k)}}{1 + \ln\left(u_\lambda\left(\frac{n}{k}\right)^{k/2}\right)}$ \\ \cline{1-1}\cline{3-3}

            $\beta_\lambda \in (1, 2)$ &  &  
            $e^{O(k)}u_\lambda^{2 - \beta_\lambda} \left(\frac{n}{k}\right)^{(\beta_\lambda - 1)k/2}$ \\ \hline

            $\beta_\lambda = 2$ & 
            $e^{O(k)} \frac{\ln(u_\lambda) + 1}{u_\lambda} \left(\frac{n}{k}\right)^k$ & 
            $e^{O(k)} \ln(u_\lambda) \left(\frac{n}{k}\right)^{k/2}$ \\ \hline

            $\beta_\lambda \in (2, 3)$ & 
            $e^{O(k)} \frac{1}{u_\lambda^{3 - \beta_\lambda}} \left(\frac{n}{k}\right)^k$ & 
            \cellcolor{green!20!lightgray!60} $e^{O(k)} \left(\frac{n}{k}\right)^{(\beta_\lambda - 1)k/2}$ \\ \hline
            
            $\beta_\lambda = 3$ & 
            $e^{O(k)} \frac{1}{\ln(u_\lambda + 1)} \left(\frac{n}{k}\right)^k$ & 
            $e^{O(k)} \left(\frac{n}{k}\right)^k / \ln\left(\left(\frac{n}{k}\right)^k\right)$ \\ \hline

            $\beta_\lambda > 3$ & 
            \multicolumn{2}{c|}{$e^{O(k)} \left(\frac{n}{k}\right)^k$} \\ \hline
		\end{tabular}
	\end{center}
\end{table}

\begin{table}[t]
    \caption{Influence of the hyperparameters $\beta_p$ and $\beta_c$ on  $p_{pc} = \Pr[p \in [\sqrt{\frac kn}, \sqrt{\frac {2k}n}] \wedge c \in [\sqrt{\frac kn}, \sqrt{\frac {2k}n}]]$ when both $u_p$ and $u_c$ are at least $\sqrt{2k}$. Asymptotical notation is used for $n \to +\infty$. The highlighted cell shows the result for the hyperparameters suggested in Corollary~\ref{cor:hyperparameters}.}
	\label{tbl:ppc}
	\begin{center}
        \small
        \begin{tabular}{|c||l|l|l|}
        \hline
        & $0 \le \beta_p < 1$ & $\beta_p = 1$ & $\beta_p > 1$ \\
        \hline
        $\beta_c < 1$ & 
        $\Theta\left(\frac{k^{(1 - (\beta_p + \beta_c)/2)}}{u_p^{1 - \beta_p}u_c^{1 - \beta_c}}\right)$ & 
        $\Theta\left(\frac{k^{(1 - \beta_c)/2}}{u_c^{1 - \beta_c}\log(u_p)}\right)$ & 
        $\Theta\left(\frac{k^{(1 - (\beta_p + \beta_c)/2)}}{u_c^{1 - \beta_c}} \right)$ \\
        \hline
        $\beta_c = 1$ & 
        $\Theta\left(\frac{k^{(1 - \beta_p)/2}}{u_p^{1 - \beta_p}\log(u_c)}\right)$ & 
        $\Theta\left(\frac{1}{\log(u_p)\log(u_c)}\right)$ & 
        $\Theta\left(\frac{k^{(1 - \beta_p)/2}}{\log(u_c)}\right)$ \\
        \hline
        $\beta_c > 1$ & 
        $\Theta\left(\frac{k^{(1 - (\beta_p + \beta_c)/2)}}{u_p^{1 - \beta_p}} \right)$ & 
        $\Theta\left(\frac{k^{(1 - \beta_c)/2}}{\log(u_p)}\right)$ & 
        \cellcolor{green!20!lightgray!60} $\Theta\left(k^{(1 - (\beta_p + \beta_c)/2)}\right)$ \\
        \hline
            
		\end{tabular}
	\end{center}
\end{table}



The proof of Theorem~\ref{thm:jump} follows from similar arguments as in~\cite[Theorem 6]{AntipovD20ppsn}, the main differences being highlighted in the following two lemmas.

\begin{lemma}\label{lem:p_pc}
    Let $k \le \frac{n}{4}$. If $u_p \ge \sqrt{2k}$ and $u_c \ge \sqrt{2k}$, then the probability $p_{pc} = \Pr[p \in [\sqrt{\frac kn}, \sqrt{\frac {2k}n}] \wedge c \in [\sqrt{\frac kn}, \sqrt{\frac {2k}n}]]$ is as shown in Table~\ref{tbl:ppc}.
\end{lemma}

\begin{proof}
    Since we choose $p$ and $c$ independently, we have 
    \[
        p_{pc} = \Pr\left[p \in \left[\sqrt{\frac kn}, \sqrt{\frac {2k}n}\right]\right] \cdot \Pr\left[c \in \left[\sqrt{\frac kn}, \sqrt{\frac {2k}n}\right]\right].
    \]
    By the definition of the power-law distribution and by Lemmas~\ref{lem:sum_estimates} and~\ref{lem:c}, we have 
    \begin{align*}
        \Pr&\left[p \in \left[\sqrt{\frac kn}, \sqrt{\frac {2k}n}\right]\right] = C_{\beta_p, u_p} \sum_{i = \lceil\sqrt{k}\rceil}^{\lfloor\sqrt{2k}\rfloor} i^{-\beta_p} \\
        &= \begin{cases}
            \Theta\left(\left(\frac{\sqrt{k}}{u_p}\right)^{1 - \beta_p}\right), &\text{ if } 0 \le \beta_p < 1 \\
            \Theta\left(\frac{1}{\log(u_p)}\right), &\text{ if } \beta_p = 1 \\
            \Theta\left(k^{\frac{1 - \beta_p}{2}}\right), &\text{ if } \beta_p > 1. \\
        \end{cases}
    \end{align*}
    We can estimate $\Pr[c \in [\sqrt{\frac kn}, \sqrt{\frac {2k}n}]]$ in the same manner, which gives us the final estimate of $p_{pc}$ shown in Table~\ref{tbl:ppc}.
\end{proof}

Now we proceed with an estimate of the probability to find the optimum in one iteration after choosing $p$ and $c$.

\begin{lemma}
    \label{lem:successful_iter}
    Let $k \in [2..\frac n4]$. Let $\lambda$, $p$ and $c$ be already chosen in an iteration of the heavy-tailed \ollga and let $p, c \in [\sqrt{\frac kn}, \sqrt{\frac {2k}n}]$. If the current individual $x$ of the heavy-tailed \ollga is in the local optimum of $\jump_k$, then the probability that the algorithm generates the global optimum in one iteration is at least $e^{-\Theta(k)}\min\{1, (\frac{k}{n})^k \lambda^2 \}$.
\end{lemma}

\begin{proof}
    The probability $P_{pc}(\lambda)$ that we find the optimum in one iteration is the probability that we have a successful mutation phase and a successful crossover phase in the same iteration. If we denote the probability of a successful mutation phase by $p_M$ and the probability of a successful crossover phase by $p_C$, then we have $P_{pc}(\lambda) = p_M p_C$.
    Then with $q_\ell$ being some constant which denotes the probability that the number $\ell$ of bits we flip in the mutation phase is in $[pn, 2pn]$, by Lemmas~3.1 and~3.2 in~\cite{AntipovDK22} we have
    \begin{align*}
        P_{pc}(\lambda) &= p_M p_C = \frac{q_\ell}{2} \min\left\{1, \lambda \left(\frac{p}{2}\right)^k\right\} \cdot \frac{1}{2} \min\left\{1, \lambda c^k (1 - c)^{2pn - k} \right\} \\
        &\ge \frac{q_\ell}{4} \min\left\{1, \lambda \left(\frac{1}{2}\sqrt{\frac{k}{n}}\right)^k\right\} \min\left\{1, \lambda \sqrt{\frac{k}{n}}^k \left(1 - \sqrt{\frac{2k}{n}}\right)^{2\sqrt{2kn}}\right\} \\
        &= \frac{q_\ell}{4} \min\left\{1, \lambda 2^{-k} \sqrt{\frac{k}{n}}^k\right\} \min\left\{1, \lambda e^{-\Theta(k)} \sqrt{\frac{k}{n}}^k \right\}
    \end{align*}
    If $\lambda \ge \sqrt{\frac{n}{k}}^k$, then we have 
    \begin{align*}
        P_{pc}(\lambda) \ge \frac{q_\ell}{4} 2^{-k} e^{-\Theta(k)} = e^{-\Theta(k)}.
    \end{align*}
    Otherwise, if $\lambda < \sqrt{\frac{n}{k}}^k$, then both minima are equal to their second argument. Thus, we have
    \begin{align*}
        P_{pc}(\lambda) \ge \frac{q_\ell}{4} \lambda^2 2^{-k} e^{-\Theta(k)} \left(\frac{k}{n}\right)^k = e^{-\Theta(k)} \left(\frac{k}{n}\right)^k \lambda^2.
    \end{align*}
    Bringing the two cases together, we finally obtain
    \begin{align*}
        P_{pc}(\lambda) \ge e^{-\Theta(k)}\min\left\{1, \left(\frac{k}{n}\right)^k \lambda^2 \right\}.
    \end{align*}
\end{proof}

Now we are in position to prove Theorem~\ref{thm:jump}.

\begin{proof}[Proof of Theorem~\ref{thm:jump}]
    Let the current individual $x$ of the heavy-tailed \ollga be already in the local optimum. Let $P$ be the probability of event $F$ when the algorithm finds optimum in one iteration. By the law of total probability this probability is at least
    \begin{align*}
        P \ge p_{(F \mid pc)} \cdot p_{pc},
    \end{align*} 
    where $p_{(F \mid pc)}$ denotes $\Pr[F \mid p, c \in [\sqrt{\frac kn}, \sqrt{\frac {2k}n}]]$ and $p_{pc}$ denotes $\Pr[p, c \in [\sqrt{\frac kn}, \sqrt{\frac {2k}n}]]$.

    The number $T_I$ of iterations until we jump to the optimum follows a geometric distribution $\Geom(P)$ with success probability $P$. Therefore,
    \begin{align*}
        E[T_I] = \frac 1P \le \frac{1}{p_{(F \mid pc)} p_{pc}}.
    \end{align*}
    
    Since in each iteration the heavy-tailed \ollga performs $2\lambda$ fitness evaluations (with $\lambda$ chosen from the power-law distribution), by Wald's equation (Lemma~\ref{lem:wald}) the expected number $E[T_F]$ of fitness evaluations the algorithm makes before it finds the optimum is
    \begin{align*}
        E[T_F] = E[T_I]E[2\lambda] \le \frac{2E[\lambda]}{p_{(F \mid pc)} \cdot p_{pc}}.
    \end{align*}

    In the remainder we show how $E[\lambda]$, $p_{(F \mid pc)}$ and $p_{pc}$ depend on the hyperparameters of the algorithm.

    First we note that $p_{pc}$ was estimated in Lemma~\ref{lem:p_pc}. Also, by Lemma~\ref{lem:expectation} the expected value of $\lambda$ is 
    \begin{align*}
        E[\lambda] = \begin{cases}
            \Theta(u_\lambda), & \text{ if } \beta_\lambda < 1, \\
            \Theta(\frac{u_\lambda}{\log(u_\lambda) + 1}), & \text{ if } \beta_\lambda = 1, \\
            \Theta(u_\lambda^{2 - \beta_\lambda}), & \text{ if } \beta_\lambda \in (1, 2), \\
            \Theta(\log(u_\lambda) + 1), & \text{ if } \beta_\lambda = 2, \\
            \Theta(1), & \text{ if } \beta_\lambda > 2.
        \end{cases}
    \end{align*}

    Finally, we compute the conditional probability of $F$ via the law of total probability.
    \begin{align*}
        p_{(F|pc)} = \sum_{i = 1}^{u_\lambda} \Pr[\lambda = i] P_{pc}(i),
    \end{align*}
    where $P_{pc}(i)$ is as defined in Lemma~\ref{lem:successful_iter}, in which it was shown that $P_{pc}(i) \ge e^{-\Theta(k)}\min\{1, (\frac{k}{n})^k i^2 \}$. We consider two cases depending on the value of $u_\lambda$.

    \textbf{Case 1: when $u_\lambda \le (\frac nk)^{k/2}$.} In this case we have $P_{pc}(i) \ge e^{-\Theta(k)}(\frac{k}{n})^k i^2$, hence
    \begin{align*}
        p_{(F|pc)} &\ge \sum_{i = 1}^{u_\lambda} C_{\beta_\lambda,u_\lambda} i^{-\beta_\lambda} e^{-\Theta(k)}\left(\frac{k}{n}\right)^k i^2 \\
        &=  e^{-\Theta(k)}\left(\frac{k}{n}\right)^k C_{\beta_\lambda, u_\lambda} \sum_{i = 1}^{u_\lambda} i^{2 - \beta_\lambda} \\
        &= e^{-\Theta(k)}\left(\frac{k}{n}\right)^k E[\lambda^2].
    \end{align*}
    By Lemma~\ref{lem:expectation-square} we estimate $E[\lambda^2]$ and obtain
    \begin{align*}
        p_{(F|pc)} \ge \begin{cases}
            e^{-\Theta(k)}\left(\frac{k}{n}\right)^k u_\lambda^2, &\text{ if } \beta_\lambda < 1, \\
            e^{-\Theta(k)}\left(\frac{k}{n}\right)^k \frac{u_\lambda^2}{\ln(u_\lambda) + 1}, &\text{ if } \beta_\lambda = 1, \\
            e^{-\Theta(k)}\left(\frac{k}{n}\right)^k u_\lambda^{3 - \beta_\lambda}, &\text{ if } \beta_\lambda \in (1, 3), \\
            e^{-\Theta(k)}\left(\frac{k}{n}\right)^k (\ln(u_\lambda) + 1), &\text{ if } \beta_\lambda = 3, \\
            e^{-\Theta(k)}\left(\frac{k}{n}\right)^k, &\text{ if } \beta_\lambda > 3. \\
        \end{cases} 
    \end{align*}

    \textbf{Case 2: when $u_\lambda > (\frac nk)^{k/2}$.} In this case we have $P_{pc}(i) \ge e^{-\Theta(k)}(\frac{k}{n})^k i^2$, when $i \le (\frac nk)^{k/2}$ and we have $P_{pc}(i) \ge e^{-\Theta(k)}$, when $i > (\frac nk)^{k/2}$. Therefore, we have
    \begin{align*}
        p_{(F \mid s)} &\ge \sum_{i = 1}^{\lfloor \left(\frac{n}{k}\right)^{k/2} \rfloor} C_{\beta_\lambda, u_\lambda} i^{-\beta_\lambda} \left(\frac{k}{n}\right)^k i^2 e^{-\Theta(k)} \\
                        &+   \sum_{i = \lfloor \left(\frac{n}{k}\right)^{k/2} \rfloor + 1}^{u_\lambda} C_{\beta_\lambda, u_\lambda} i^{-\beta_\lambda} e^{-\Theta(k)} \\
                        &=   C_{\beta_\lambda, u_\lambda} e^{-\Theta(k)} \left(\left(\frac{k}{n}\right)^k \sum_{i = 1}^{\lfloor \left(\frac{n}{k}\right)^{k/2} \rfloor} i^{2 - \beta_\lambda} + \sum_{i = \lfloor \left(\frac{n}{k}\right)^{k/2} \rfloor + 1}^{u_\lambda} i^{-\beta_\lambda}\right).
    \end{align*}

    Estimating the sums via Lemma~\ref{lem:sum_estimates}, we obtain
    \begin{align*}
        p_{(F|pc)} \ge \begin{cases}
            e^{-\Theta(k)}, & \text{ if } \beta_\lambda < 1, \\
            e^{-\Theta(k)}\left(1 + \ln\left(u_\lambda \left(\frac kn\right)^{k/2}\right)\right)\frac{1}{\ln(u) + 1}, & \text{ if } \beta_\lambda = 1, \\
            e^{-\Theta(k)} \left(\frac kn\right)^{(\beta - 1)k/2}, & \text{ if } \beta_\lambda \in (1, 3), \\
            e^{-\Theta(k)}\left(\frac kn\right)^k \ln\left(\left(\frac nk\right)^k\right), & \text{ if } \beta_\lambda = 3, \\
            e^{-\Theta(k)}\left(\frac kn\right)^k, & \text{ if } \beta_\lambda > 3, \\
        \end{cases}  
    \end{align*}
    
    Gathering the estimates for the two cases and the estimates of $E[\lambda]$ and $p_{pc}$ together, we obtain the runtimes listed in Table~\ref{tbl:runtime}.
\end{proof}


\subsection{Recommended Hyperparameters}
\label{sec:recommendations}

In this subsection we subsume the results of our runtime analysis to show most preferable parameters of the power-law distributions for the practical use. We point out the runtime with such parameters on \onemax and $\jump_k$ in Corollary~\ref{cor:hyperparameters}. We then also prove a lower bound on the runtime of the \ollga with static parameters to show that when $k$ is constant (that is, the most interesting case, since only then we have a polynomial runtime), then the performance of the heavy-tailed \ollga is asymptotically better than the best performance we can obtain with the static parameters. 

\begin{corollary}\label{cor:hyperparameters}
    Let $\beta_\lambda = 2 + \eps_\lambda$ and $\beta_p = 1 + \eps_p$ and $\beta_c = 1 + \eps_c$, where $\eps_\lambda, \eps_p, \eps_c > 0$ are some constants. Let also $u_\lambda$ be at least $2^n$ and $u_p = u_c = \sqrt{n}$. Then the expected runtime of the heavy-tailed \ollga is $O(n \log(n))$ fitness evaluations on \onemax and $e^{O(k)}(\frac nk)^{(1 + \eps_\lambda)k/2}$ fitness evaluations on $\jump_k$, $k \in [2..\frac{n}{4}]$. 
\end{corollary}

This corollary follows from Theorems~\ref{thm:onemax} and~\ref{thm:jump}. We only note that for the runtime on $\jump_k$ the same arguments as in Theorem~\ref{thm:onemax} show us that the runtime until we reach the local optimum is at most $O(n\log(n))$, which is small compared to the runtime until we reach the global optimum. Also we note that when $\beta_p$ and $\beta_c$ are both greater than one and $u_p = u_c =\sqrt{n} \ge \sqrt{2k}$, by Lemma~\ref{lem:p_pc} we have $p_{pc} = \Theta(k^{-\frac{\eps_p + \eps_c}{2}})$, which is implicitly hidden in the $e^{O(k)}$ factor of the runtime on $\jump_k$. We also note that $u_\lambda = 2^n$ guarantees that $u_\lambda > (\frac{n}{k})^{k/2}$, which yields the runtimes shown in the right column of Table~\ref{tbl:runtime}.

Corollary~\ref{cor:hyperparameters} shows that when we have (almost) unbounded distributions and use power-law exponents slightly greater than one for all parameters except the population size, for which we use a power-law exponent slightly greater than two, we have a good performance both on easy monotone functions, which give us a clear signal towards the optimum, and on the much harder jump functions, without any knowledge of the jump size. 


We now also show that the proposed choice of the hyper-parameters gives us a better performance than any static parameters choice on $\jump_k$ for constant $k$. As we have already noted in the introduction, only for such values of $k$ different variants of the \ollga and many other classic EAs have a polynomial runtime, hence this case is the most interesting to consider. We prove the following theorem which holds for any static parameters choice of the \ollga, even when we use different population sizes $\lambda_M$ and $\lambda_C$ in the mutation and in the crossover phases respectively.

\begin{theorem}\label{thm:static-lower}
    Let $n$ be sufficiently large. Then the expected runtime of the \ollga with any static parameters $p$, $c$, $\lambda_M$ and $\lambda_C$ on $\jump_k$ with $k \le \frac{n}{512}$ is at least $B \coloneqq \frac{1}{91\sqrt{\ln(n/k)}}(\frac{2n}{k})^{(k + 1)/2}$.
\end{theorem}

Before we prove Theorem~\ref{thm:static-lower}, we give a short sketch of the proof to ease the further reading. First we show that with high probability the \ollga with static parameters starts at a point with approximately $\frac{n}{2}$ one-bits. In the second step we handle a wide range of parameter settings and show that for them we cannot obtain a runtime better than $B$ by showing that the probability to find the optimum in one iteration is at most $1/B$. For the remaining settings we then show that we are not likely to observe an $\Omega(n)$ progress in one iteration, hence with high probability there is an iteration when we have a fitness which is $\frac{n}{2} + \Omega(n)$ and at the same time which is $n - k - \Omega(n)$. From that point on the probability that we have a progress which is $\Omega(k\log(\frac{n}{k}))$ is very unlikely to happen hence with high probability the \ollga does not reach the local optima of $\jump_k$ (nor the global one) in $\Omega(\frac{n}{k\log(\frac{n}{k})})$ iterations which is equal to $\Omega(\frac{(\lambda_M + \lambda_C)n}{k\log(\frac{n}{k})})$ fitness evaluations by the definition of the algorithm. For the narrowed range of parameters this yields the lower bound.

To transform these informal arguments into a rigorous proof we use several auxiliary tools. The first of them is Lemma 14 from~\cite{DoerrWY21}, which we formulate as follows\footnote{Note that in~\cite{DoerrWY21} the authors prove upper bounds on both getting a too high and a too low number of one-bits after applying a standard bit mutation. Since we only use the first one, we do not mention the second bound here}.

\begin{lemma}[Lemma 14 in~\cite{DoerrWY21}]\label{lem:mutation-res}
    Let $x$ be a bit string of length $n$ with exactly $m$ one-bits in it. Let $y$ be an offspring of $x$ obtained by flipping each bit independently  with probability $\frac rn$, where $r \le \frac{n}{2}$. Let also $m'$ be a random variable denoting the number of one-bits in $y$. Then for any $\Delta \ge 0$ we have
    \begin{align*}
        \Pr\left[m' - m \ge (n - 2m)\frac{r}{n} + \Delta\right] \le \exp\left(\frac{-\Delta^2}{2(1 - r/n)(r + \Delta/3)}\right).
    \end{align*}
\end{lemma}

We also use the following lemma, which bounds the probability to make a jump to a certain point.

\begin{lemma}\label{lem:prob_jump}
    If we are in distance $d \le \frac{n}{2}$ from the unique optimum of any function, then the probability $P$ that the \ollga with mutation rate $p$, crossover bias $c$ and population sizes for the mutation and crossover phases $\lambda_M$ and $\lambda_C$ respectively finds the optimum in one iteration is at most
    \begin{align*}
        P &\le \min\left\{1, \lambda_M p^d (1 - p)^{n - d} + \lambda_M \lambda_C (pc)^d(1 - pc)^{n - d}\right\} \\
        &\le \min\left\{1, 2\lambda_M \lambda_C \left(\frac{d}{2n}\right)^d\right\}.
    \end{align*}
\end{lemma}

A very similar, but less general result has been proven in~\cite{AntipovDK22} (Theorem~16). 

\begin{proof}[Proof of Lemma~\ref{lem:prob_jump}]
    Without loss of generality we assume that the unique optimum is the all-ones bit string. Hence, the current individual has exactly $d$ zero-bits. Let $p_\ell$ be the probability that we choose $\ell$ as the number of bits to flip at the start of the iteration of the \ollga. Let also $p_m(\ell)$ be the probability (conditional on the chosen $\ell$) that the mutation winner has all zero-bits flipped to ones. Note that this is necessary for crossover to be able to create the global optimum. Let $p_c(\ell)$ be the probability that conditional on the chosen $\ell$ and on that we flip all $d$ zero-bits in the mutation winner, we then flip $\ell - d$ zeros in the mutation winner in at least one crossover offspring. Then by the law of total probability we have
    \begin{align}\label{eq:P}
        P = \sum_{\ell = 0}^n p_\ell p_m(\ell) p_c(\ell).
    \end{align}

    For $\ell < d$ the probability that we flip all $d$ zero-bits in the mutation winner is zero. For $\ell = d$ we flip all $d$ zero-bits in one particular mutation offspring with probability $q_m(\ell) = \binom{n}{d}^{-1}$. Since we create all $\lambda_M$ offspring independently, the probability that we flip all $d$ zero-bits in at least one offspring is
    \begin{align*}
        p_m(\ell) = 1 - (1 - q_m(\ell))^{\lambda_M} \le \lambda_M q_m(\ell) = \lambda_M \binom{n}{d}^{-1},
    \end{align*} 
    where we used Bernoulli inequality. Since when we create such an offspring in the mutation phase, we already find the optimum, we assume that we do not need to perform the crossover and therefore, $p_c(\ell) = 1$ in this case.

    When $\ell > d$, the probability to flip all $d$ zero-bits in one offspring is 
    \begin{align*}
        q_m(\ell) = \binom{n - d}{\ell - d} \binom{n}{\ell}^{-1}.
    \end{align*}
    The probability to do so in one of $\lambda_M$ independently created offspring is thus
    \begin{align*}
        p_m = 1 - (1 - q_m(\ell))^{\lambda_M} \le \lambda_M q_m(\ell) = \lambda_M \binom{n - d}{\ell - d} \binom{n}{\ell}^{-1}.
    \end{align*}
    The probability that in one crossover offspring we take from the current individual all $\ell - d$ bits which are zeros in the mutation winner and take from the mutation winner all $d$ bits which are zeros in the current individual is $q_c(\ell) = c^d (1 - c)^{\ell - d}$. Consequently, the probability that we do this in at least one of $\lambda_C$ independently created individuals is
    \begin{align*}
        p_c(\ell) = 1 - (1 - q_c(\ell))^{\lambda_C} \le \lambda_C q_c(\ell) = \lambda_C c^d (1 - c)^{\ell - d}.
    \end{align*}

    Recall that $\ell$ is chosen from the binomial distribution $\Bin(n, p)$, thus we have $p_\ell = \binom{n}{\ell}p^{\ell}(1 - p)^{n - \ell}$. Putting all the estimates above into~\eqref{eq:P} we obtain
    \begin{align*}
        P &= p_d p_m(d) + \sum_{\ell = d + 1}^n p_\ell p_m(\ell) p_c(\ell) \\
        &\le \binom{n}{d} p^d (1 - p)^{n - d} \lambda_M \binom{n}{d}^{-1} \\
        &+ \sum_{\ell = d + 1}^n \binom{n}{\ell} p^\ell (1 - p)^{n - \ell} \lambda_M \binom{n - d}{\ell - d} \binom{n}{\ell}^{-1} \lambda_C c^d (1 - c)^{\ell - d} \\
        &= \lambda_M p^d (1 - p)^{n - d} + \lambda_M \lambda_C (1 - p)^{n - d} (pc)^d \sum_{\ell = d + 1}^n \binom{n - d}{\ell - d} \left(\frac{p(1 - c)}{1 - p}\right)^{\ell - d} \\
        &= \lambda_M p^d (1 - p)^{n - d} + \lambda_M \lambda_C (1 - p)^{n - d} (pc)^d \sum_{i = 1}^{n - d} \binom{n - d}{i} \left(\frac{p(1 - c)}{1 - p}\right)^{i} \\
        &\le \lambda_M p^d (1 - p)^{n - d} + \lambda_M \lambda_C (1 - p)^{n - d} (pc)^d \left(\frac{p(1 - c)}{1 - p} + 1\right)^{n - d} \\
        &= \lambda_M p^d (1 - p)^{n - d} + \lambda_M \lambda_C (1 - p)^{n - d} (pc)^d \left(\frac{1 - pc}{1 - p}\right)^{n - d} \\
        &= \lambda_M p^d (1 - p)^{n - d} + \lambda_M \lambda_C (1 - pc)^{n - d} (pc)^d.
    \end{align*}
    
    We now consider function $f_d(x) = x^d (1 - x)^{n - d}$ on interval $x \in [0, 1]$. To find its maximum, we consider its value in the ends of the interval (which is zero in both ends) and in the roots of its derivative, which is
    \begin{align*}
        f_d'(x) = dx^{d - 1} (1 - x)^{n - d} - (n - d) x^d (1 - x)^{n - d - 1} = x^{d - 1} (1 - x)^{n - d - 1} (d - nx).
    \end{align*} 
    Hence, the only root of the derivative is in $x = \frac{d}{n}$. Since $f_d(x)$ is a smooth function, it reaches its maximum there, which is,
    \begin{align*}
        f_d\left(\frac{d}{n}\right) = \left(\frac{d}{n}\right)^d \left(1 - \frac{d}{n}\right)^{n - d}.
    \end{align*}
    Since we assume that $d \le \frac{n}{2}$, we conclude that for all $x \in [0, 1]$ we have
    \begin{align*}
        f_d(x) &\le \left(\frac{d}{n}\right)^d \left(1 - \frac{d}{n}\right)^{n - d} = \left(\frac{d}{n}\right)^d \left(\left(1 - \frac{d}{n}\right)^{\frac{n}{d} - 1}\right)^d \le \left(\frac{d}{2n}\right)^d.
    \end{align*}

    Hence we have both $p^d (1 - p)^{n - d} \le (\frac{d}{2n})^d$ and $(1 - pc)^{n - d} (pc)^d \le (\frac{d}{2n})^d$, from which we conclude
    \begin{align*}
        P &\le \lambda_M p^d (1 - p)^{n - d} + \lambda_M \lambda_C (1 - pc)^{n - d} (pc)^d \\
        &\le (\lambda_M + \lambda_M \lambda_C)\left(\frac{d}{2n}\right)^d \le 2\lambda_M \lambda_C \left(\frac{d}{2n}\right)^d.  
    \end{align*}

    Since $P$ cannot exceed one, we also have
    \begin{align*}
        P &\le \min\left\{1, \lambda_M p^d (1 - p)^{n - d} + \lambda_M \lambda_C (1 - pc)^{n - d} (pc)^d\right\} \\
        &\le \min\left\{1, 2\lambda_M \lambda_C \left(\frac{d}{2n}\right)^d\right\}. \qedhere
    \end{align*} 

\end{proof}

An important corollary from Lemma~\ref{lem:prob_jump} is the following lower bound for the case when we use too small population sizes.

\begin{corollary}\label{cor:pop_sizes}
    Consider the run of the \ollga with static parameters on $\jump_k$ with $k < \frac{n}{2}$. Let the population sizes which are used for the mutation and crossover phases be $\lambda_M$ and $\lambda_C$ respectively. Let also the current individual $x$ be a point outside the fitness valley, but with at least $\frac{n}{2}$ one-bits. Then if $\lambda_M\lambda_C < \ln(\frac{n}{k})(\frac{2n}{k})^{k - 1}$, then the expected runtime until we find the global optimum is at least $\frac{1}{2\sqrt{\ln(n/k)}}(\frac{2n}{k})^{\frac{k + 1}{2}}$.
\end{corollary}

\begin{proof}
    Since the algorithm has already found the point outside the fitness valley, it will never accept a point inside it as the current individual $x$. Hence, unless we find the optimum, the distance to it from the current individual is at least $k$ and at most $\frac{n}{2}$.

    We now consider the term $(\frac{d}{2n})^d$, which is used in the bound given in Lemma~\ref{lem:prob_jump}, as a function of $d$ and maximize it for $d \in [k, \frac{n}{2}]$. For this purpose we consider its values in the ends of the interval and in the zeros of its derivative, which is,
    \begin{align*}
        \left(\left(\frac{d}{2n}\right)^d\right)' &= \left(\frac{d}{2n}\right)^d \left(d \ln\left(\frac{d}{2n}\right)\right)' \\
        &= \left(\frac{d}{2n}\right)^d \left(\ln\left(\frac{d}{2n}\right) + 1\right).
    \end{align*}
    Hence, the derivative is equal to zero only when $\frac{d}{2n} = e^{-1}$, that is, when $d = \frac{2n}{e}$. Since we only consider $d$ which are at most $\frac{n}{2}$, the derivative does not have roots in this range. We also note that for $d < \frac{2n}{e}$ the derivative is negative, hence the maximal value of $(\frac{d}{2n})^d$ is reached when $d = k$. Therefore, by Lemma~\ref{lem:prob_jump} we have
    \begin{align}\label{eq:P-general}
        \begin{split}
            P &\le \min\left\{1, 2\lambda_M\lambda_C \left(\frac{d}{2n}\right)^d\right\} \\
              &\le \min\left\{1, 2\lambda_M\lambda_C \left(\frac{k}{2n}\right)^k\right\}.
        \end{split}
    \end{align}

    
    Since $\lambda_M\lambda_C \le \ln(\frac{n}{k})(\frac{2n}{k})^{k - 1}$ and since for all $x \ge 2$ we have $\ln(x) < \frac{x}{2}$, we compute
    \begin{align*}
        2\lambda_M\lambda_C \left(\frac{k}{2n}\right)^k \le 2\ln\left(\frac{n}{k}\right)\left(\frac{2n}{k}\right)^{k - 1} \left(\frac{k}{2n}\right)^k = \ln\left(\frac{n}{k}\right) \cdot \frac{k}{n} \le \frac{1}{2}.
    \end{align*}
    Therefore, the minimum in~\eqref{eq:P-general} is equal to the second argument.
    
    Thus, the runtime $T_I$ (in terms of iterations) is dominated by the geometric distribution with parameter $2\lambda_M\lambda_C \left(\frac{k}{2n}\right)^k \le \frac{1}{2}$. This implies that the expected number of unsuccessful iterations is $E[T_I] - 1 \ge \frac{E[T_I]}{2}$.
    Since in each unsuccessful iteration we have exactly $\lambda_M + \lambda_C$ fitness evaluation, we have
    \begin{align*}
        E[T_F] &= (\lambda_M + \lambda_C) \frac{E[T_I]}{2} \ge \frac{\lambda_M + \lambda_C}{4\lambda_M\lambda_C \left(\frac{k}{2n}\right)^k} = \left(\frac{1}{\lambda_M} + \frac{1}{\lambda_C}\right) \cdot \frac{1}{4} \left(\frac{2n}{k}\right)^k.
    \end{align*}
    By Lemma~\ref{lem:a-g-means} we obtain
    \begin{align*}
        E[T_F] &\ge \left(\frac{1}{\lambda_M} + \frac{1}{\lambda_C}\right) \cdot \frac{1}{4} \left(\frac{2n}{k}\right)^k \ge \frac{1}{2}\sqrt{\frac{1}{\lambda_M\lambda_C}} \left(\frac{2n}{k}\right)^k \\
        &\ge \frac{1}{2}\sqrt{\frac{1}{\ln\left(\frac{n}{k}\right)}\left(\frac{k}{2n}\right)^{k - 1}} \cdot \left(\frac{2n}{k}\right)^k  = \frac{1}{2\sqrt{\ln\left(\frac{n}{k}\right)}}\left(\frac{2n}{k}\right)^{\frac{k + 1}{2}}.
    \end{align*}

\end{proof}

In the following lemma we also show that too large population sizes also yield a too large expected runtime.

\begin{lemma}\label{lem:large-lambda}
    Consider the run of the \ollga with static parameters on $\jump_k$ with $k < \frac{n}{2}$. Let the population sizes which are used for the mutation and crossover phases be $\lambda_M$ and $\lambda_C$ respectively. Let also the current individual $x$ be a point outside the fitness valley, but with at least $\frac{n}{2}$ one-bits. Then if $\lambda_M\lambda_C > \frac{1}{\ln(\frac{n}{k})}(\frac{2n}{k})^{k + 1}$, then the expected runtime until we find the global optimum is at least $\frac{1}{16\sqrt{\ln(n/k)}}(\frac{2n}{k})^{\frac{k + 1}{2}}$.
\end{lemma}
\begin{proof}
    By Lemma~\ref{lem:a-g-means} we have that the cost of one iteration is
    \begin{align*}
        \lambda_M + \lambda_C \ge 2\sqrt{\lambda_M\lambda_C} > \frac{2}{\sqrt{\ln\left(\frac{n}{k}\right)}}\left(\frac{2n}{k}\right)^{\frac{k + 1}{2}},
    \end{align*}
    hence to prove this lemma it is enough to consider only the first iteration of the algorithm, which already takes at least $\frac{2}{\sqrt{\ln(n/k)}}(\frac{2n}{k})^{\frac{k + 1}{2}}$ fitness evaluations plus one evaluation for the initial individual.

    We first show that if $\lambda_M \ge 2^{\frac{n}{2}} - 2$, then we are not likely to sample the optimum before making $2^{\frac{n}{2}} - 1$ fitness evaluations. For this we note that the initial individual and all mutation offspring in the first iteration are sampled independently of the fitness function, thus they are random points in the search space. Therefore, for each of these individuals the probability to be the optimum is $2^{-n}$. Consequently, by the union bound, when we create the initial individual and $2^{\frac{n}{2}} - 2$ mutation offspring, the probability that at least one of them is the optimum is at most 
    \begin{align*}
        \frac{2^{\frac{n}{2}} - 1}{2^n} = 2^{-\frac{n}{2}}(1 - 2^{-n}) \le 2^{-\left(\frac{n}{2} + 1\right)}.
    \end{align*}
    Hence, with probability at least $(1 - 2^{-\left(\frac{n}{2} + 1\right)})$ we have to make $2^\frac{n}{2} - 1$ or more fitness evaluations, which implies that
    \begin{align*}
        E[T_F] \ge \left(1 - 2^{-\left(\frac{n}{2} + 1\right)}\right) \left(2^{\frac{n}{2}} - 1\right) = 2^\frac{n}{2} - \frac{3}{2} + 2^{-\left(\frac{n}{2} + 1\right)} \ge 2^{\frac{n}{2} - 1},
    \end{align*}
    if $n \ge 3$. Without proof we note that $\frac{1}{16\sqrt{\ln(n/k)}}(\frac{2n}{k})^{\frac{k + 1}{2}}$ is increasing in $k$ for $k \le \frac{n}{2}$. Hence, for $k \le \frac{n}{2}$ we have that 
    \begin{align*}
        \frac{1}{16\sqrt{\ln(n/k)}}\left(\frac{2n}{k}\right)^{\frac{k + 1}{2}} \le \frac{1}{16\sqrt{\ln(2)}} \cdot 4^{\left(\frac{n}{4} + \frac{1}{2}\right)} \le 2^{\frac{n}{2} - 2} \le E[T_F].
    \end{align*}

    In the rest of the proof we assume that $\lambda_M < 2^{\frac{n}{2}} - 2$. Since the mutation winner is chosen based on the fitness, we cannot use the same argument with random points in the search space for the crossover phase. However, we can consider an artificial process, which in parallel runs the crossover phase for each mutation offspring seen as winner. If none of these parallel processes has generated the optimum within $m$ crossover offspring samples, then also the true process has not done so within a total of $1 + \lambda_M + m$ fitness evaluations. We note that in the parallel crossover phases, since no selection has been made, again all offspring are uniformly distributed in $\{0,1\}^n$.

    Let us fix $m = 2^{\frac{n}{2} - 1}$. By the union bound, the probability that one of $1 + \lambda_M + m\lambda_M$ individuals generated by the artificial process is the optimum is at most
    \begin{align*}
        \frac{1 + \lambda_M + m\lambda_M}{2^n} &< \frac{1 + \left(2^{\frac{n}{2}} - 2\right)(1 + m)}{2^n} = \frac{1 + \left(2^{\frac{n}{2}} - 2\right)\left(2^{\frac{n}{2} - 1} + 1\right)}{2^n} = \frac{1 + \frac{1}{2}\left(2^n - 4\right)}{2^n} \le \frac{1}{2}.
    \end{align*} 
    At the same time, if the original \ollga creates $1 + \lambda_M + m \ge m$ individuals, it also performs at least $m$ fitness evaluations. Hence, the expected number of fitness evaluations is at least
    \begin{align*}
        E[T_F] \ge \frac{m}{2} = 2^{\frac{n}{2} - 2} \ge \frac{1}{16\sqrt{\ln(n/k)}}\left(\frac{2n}{k}\right)^{\frac{k + 1}{2}}.
    \end{align*}
\end{proof}
 
We are now in position to prove Theorem~\ref{thm:static-lower}.

\begin{proof}[Proof of Theorem~\ref{thm:static-lower}]
    \textbf{Initialization.} Recall that the initial individual is sampled uniformly at random, hence the number of one-bits in it follows a binomial distribution $\Bin(n, \frac{1}{2})$. By the symmetry argument we have that the number of one-bits $X$ in the initial individual is at least $\frac{n}{2}$ with probability at least $\frac{1}{2}$. By Chernoff bounds (see, e.g., Theorem~1.10.1 in~\cite{Doerr20bookchapter}) we also have that the probability that $X$ is greater than $\frac{n}{2} + \frac{n}{8}$ is at most
    \begin{align*}
        \Pr\left[X \ge \left(1 + \frac{1}{4}\right) \frac{n}{2}\right] \le \exp\left(-\frac{n/2}{48}\right) = e^{-\Theta(n)}.
    \end{align*}
    Hence, with probability at least $\frac{1}{2} - e^{-\Theta(n)}$ the initial individual has a number of one-bits (and hence, the fitness) in $[\frac{n}{2}, \frac{5n}{8}]$. We now condition on this event\footnote{Without proof we note that if the initial individual has less than $\frac{n}{2}$ one-bits, our lower bound would also hold. However, the proof of this fact would require more complicated arguments, hence in order to increase the readability of the paper we avoid considering that case.}.

    \textbf{Narrowing the reasonable population sizes.} Since we condition on starting in distance $d \le \frac{n}{2}$ from the optimum of $\jump_k$, by Corollary~\ref{cor:pop_sizes} and Lemma~\ref{lem:large-lambda} we have that if we choose $\lambda_M$ and $\lambda_C$ such that $\lambda_M\lambda_C \ge \frac{1}{\ln(n/k)}(\frac{2n}{k})^{k + 1}$ or $\lambda_M\lambda_C \le \ln(\frac{n}{k})(\frac{2n}{k})^{k - 1}$, then the expected runtime is at least$\frac{1}{16\sqrt{\ln(n/k)}}(\frac{2n}{k})^{\frac{k + 1}{2}} = \frac{91B}{16}$. Hence, in the rest of the proof we assume that $\ln(\frac{n}{k})(\frac{2n}{k})^{k - 1} < \lambda_M \lambda_C < \frac{1}{\ln(n/k)}(\frac{2n}{k})^{k + 1}$.

    We note that by Lemma~\ref{lem:a-g-means} this assumption also implies that the cost of one iteration is
    \begin{align*}
        \lambda_M + \lambda_C \ge 2\sqrt{\lambda_M\lambda_C} \ge 2 \sqrt{\ln\left(\frac{n}{k}\right)}\left(\frac{2n}{k}\right)^{\frac{k - 1}{2}}.
    \end{align*}

    \textbf{Narrowing the reasonable mutation rate and crossover bias.} We now show that using a too large mutation rate or crossover bias also yields a runtime which is greater than $(\frac{2n}{k})^{\frac{k + 1}{2}}$ and therefore greater than $B$. Conditional on the current individual $x$ being in distance $d \le \frac{n}{2}$ from the optimum, by Lemma~\ref{lem:prob_jump} we have that if $pc \ge \frac{1}{2}$ (and therefore, $p \ge \frac{1}{2}$), then we have
    \begin{align*}
        P \le \lambda_M (1 - p)^{n/2} + \lambda_M \lambda_C (1 - pc)^{n/2} \le \lambda_M \left(\frac{1}{2}\right)^\frac{n}{2} + \lambda_M \lambda_C \left(\frac{1}{2}\right)^\frac{n}{2} \le 2\lambda_M \lambda_C \left(\frac{1}{2}\right)^\frac{n}{2}.
    \end{align*}
    Therefore, the expected number of fitness evaluations until we find the optimum is at least
    \begin{align*}
        E[T_F] \ge \frac{\lambda_M + \lambda_C}{P} \ge \left(\frac{1}{\lambda_M} + \frac{1}{\lambda_C}\right) \cdot 2^{\frac{n}{2} - 1}.
    \end{align*}
    Since we already assume that $\lambda_M \lambda_C \le \frac{1}{\ln(n/k)}(\frac{2n}{k})^{k + 1}$, by Lemma~\ref{lem:a-g-means} we have

    \begin{align*}
        \frac{1}{\lambda_M} + \frac{1}{\lambda_C} &\ge 2 \sqrt{\frac{1}{\lambda_M\lambda_C}} \ge  2 \sqrt{\ln\left(\frac{n}{k}\right)}\left(\frac{k}{2n}\right)^{\frac{k + 1}{2}}.
    \end{align*}
    Therefore, we have
    \begin{align*}
        E[T_F] \ge 2^{\frac{n}{2}} \ln\left(\frac{n}{k}\right) \left(\frac{k}{2n}\right)^{\frac{k + 1}{2}} = 2^{\frac{n}{2}} \ln\left(\frac{n}{k}\right) \left(\frac{k}{2n}\right)^{k + 1} \left(\frac{2n}{k}\right)^{\frac{k + 1}{2}}.
    \end{align*}
    We note that for $k \in [1, \frac{n}{512}]$ the term $\left(\frac{k}{2n}\right)^{k + 1}$ is decreasing in $k$ (we avoid the proof of this fact, but note that it trivially follows from considering the derivative). Consequently, if we assume that $k \le \frac{n}{512}$, then we have 
    \begin{align*}
        E[T_F] &\ge 2^\frac{n}{2} \ln\left(512\right) \left(\frac{1}{1024}\right)^{\frac{n}{512} + 1} \left(\frac{2n}{k}\right)^{\frac{k + 1}{2}} = 2^{\frac{n}{2} - 10(\frac{n}{512} + 1)} \cdot \ln(512) \left(\frac{2n}{k}\right)^{\frac{k + 1}{2}} \\
        &= \frac{\ln(512)}{1024} \cdot 2^{\frac{123n}{256}}\left(\frac{2n}{k}\right)^{\frac{k + 1}{2}} \ge \frac{\ln(512)}{1024} \cdot 2^{\frac{246}{256}}\left(\frac{2n}{k}\right)^{\frac{k + 1}{2}} \ge \frac{1}{91}\left(\frac{2n}{k}\right)^{\frac{k + 1}{2}}.
    \end{align*}
    Hence, using $pc \ge \frac{1}{2}$ gives us the expected runtime which is not less than $B$.

    \textbf{Making a linear progress.} For the rest of the proof we assume that we have population sizes such that $\lambda_M \lambda_C \in [\ln(\frac{n}{k})(\frac{2n}{k})^{k - 1}, \frac{1}{\ln(n/k)}(\frac{2n}{k})^{k + 1}]$ and $p$ and $c$ such that $pc \le \frac{1}{2}$. We now show that at some iteration before we have already made at least $(\frac{2n}{k})^{\frac{k + 1}{2}}$ fitness evaluations we get a current individual $x$ with fitness in $[\frac{n}{2} + \frac{n}{8}, \frac{n}{2} + \frac{n}{4}]$. For this we show that conditional on $f(x) \ge \frac{n}{2}$ we are not likely to increase fitness in one iteration by at least $\frac{n}{8}$ in a very long time.
    
    For this purpose we consider a modified iteration of the \ollga, where in the crossover phase we create not only $\lambda_C$ offspring by crossing the current individual $x$ with the mutation winner $x'$, but we create $\lambda_M \cdot \lambda_C$ offspring by performing crossover between $x$ and each mutation offspring $\lambda_C$ times. The best offspring in this modified iteration cannot be worse than the best offspring in a non-modified iteration. Hence the probability that we increase the fitness by a least $\frac{n}{8}$ is at most the probability that the best offspring of this modified iteration is better than the current individual $x$ by at least $\frac{n}{8}$.

    Consider one particular offspring $y'$ created in this modified iteration. Recall that its parent was created by first choosing a number $\ell$ from the binomial distribution $\Bin(n, p)$ and then flipping $\ell$ bits, therefore it is distributed as if we created it by flipping each bit independently with probability $p$. Then when we create $y'$ we take each flipped bit from its parent with probability $c$, hence in the resulting offspring each bit is flipped with probability $pc$, independently of other bits. Consequently the distribution of $y'$ is the same as if we created it via the standard bit mutation with probability of flipping each bit equal to $pc$. Note that this argument works only when we consider one particular individual, since the mutation offspring are dependent on each other (since they have the same number $\ell$ of bits flipped) and therefore, their offspring are all also dependent.

    To estimate the probability that $y'$ has a fitness by $\frac n8$ greater than $x$, we use Lemma~\ref{lem:mutation-res} with $r = pcn$ (note that since $pc < \frac{1}{2}$, we have $r \le \frac{n}{2}$, thus we satisfy the conditions of Lemma~\ref{lem:mutation-res}). Since we are conditioning on $m = f(x) \ge \frac{n}{2}$, we have $(n - 2m)\frac{r}{n} \le 0$. Hence, with $\Delta = \frac{n}{8}$ we obtain
    \begin{align*}
        \Pr\left[f(y') - f(x) \ge \frac{n}{8}\right] &\le \Pr\left[f(y') - f(x) \ge (n - 2f(x))\frac{r}{n} + \Delta\right] \\
        &\le \exp\left(\frac{-\Delta^2}{2(1 - r/n)(r + \Delta/3)}\right) \\
        &\le \exp\left(\frac{-\left(\frac{n}{8}\right)^2}{2\left(pcn + \frac{n}{24}\right)}\right) \le \exp\left(\frac{-\left(\frac{n}{8}\right)^2}{2\left(\frac{n}{2} + \frac{n}{24}\right)}\right)\\
        &\le \exp\left(-\frac{n^2}{64} \cdot \frac{12}{13n}\right) = e^{-\frac{3n}{208}} \le e^{-\frac{n}{70}}.
    \end{align*}

    After $\frac{n}{k}$ modified iterations we create $\frac{\lambda_M \lambda_C n}{k}$ offspring, therefore by the union bound the probability that at least one of them has a fitness by at least $\frac{n}{8}$ greater than the fitness of its parent is at most $\frac{\lambda_M \lambda_C n}{k}e^{-\frac{n}{70}}$. Since we also have $\lambda_M \lambda_C \le \frac{1}{\ln(n/k)}(\frac{2n}{k})^{k + 1}$, this probability is at most
    \begin{align*}
        \frac{1}{\ln\left(\frac{n}{k}\right)} \left(\frac{2n}{k}\right)^{k + 1} \cdot \frac{n}{k} \cdot e^{-\frac{3n}{16}} = \frac{1}{2\ln\left(\frac{n}{k}\right)}\left(\frac{2n}{k}\right)^{k + 2} e^{-\frac{n}{70}} \le \frac{1}{2}\left(\frac{2n}{k}\right)^{k + 2} e^{-\frac{n}{70}}.
    \end{align*}
    We also note that the term $\left(\frac{2n}{k}\right)^{k + 2}$ is increasing in $k$ for $k \in [1, \frac{n}{512}]$ (we omit the proof, since it trivially follows from considering its derivative). Therefore, for such $k$ this probability is at most
    \begin{align*}
        \frac{1}{2} \cdot 1024^{\frac{n}{512} + 2} \cdot e^{-\frac{n}{70}} = 2^{19} \exp\left(\frac{\ln(1024)n}{512} - \frac{n}{70}\right) \le 2^{19} e^{-0.0007n} = e^{-\Theta(n)}.
    \end{align*}

    Let $T_{5n/8}$ be the first iteration when we have $x$ with at least $\frac{5n}{8}$ one-bits. If $T_{5n/8} \le \frac{n}{k}$, then with probability $1 - e^{-\Theta(n)}$ none of the offspring created up to this moment improved the fitness by more than $\frac{n}{8}$. Hence, we have that $x$ has at most $\frac{5n}{8} + \frac{n}{8} = \frac{3n}{4}$ one-bits. Otherwise, if $T_{5n/8} > \frac{n}{k}$, by this iteration we already make at least
    \begin{align*}
        (\lambda_M + \lambda_C) \frac{n}{k} \ge 2 \sqrt{\ln\left(\frac{n}{k}\right)}\left(\frac{2n}{k}\right)^{\frac{k - 1}{2}} \frac{n}{k} \ge \left(\frac{2n}{k}\right)^{\frac{k + 1}{2}}.
    \end{align*}
    fitness evaluations and thus the runtime exceeds $(\frac{2n}{k})^{\frac{k + 1}{2}}$. Hence, in the rest of the proof we assume that at some point we have a current individual $x$ with fitness in $[\frac{5n}{8}, \frac{3n}{4}]$.

    \textbf{Slow progress towards the local optimum.}
    We now show that after reaching fitness at least $\frac{5n}{8}$, the \ollga makes a progress not greater than $\delta \coloneqq \frac{26}{3}(k + 2) \ln(\frac{n}{k})$ per iteration.
    
    For this purpose we again consider an iteration of the modified algorithm, which generates $\lambda_M\lambda_C$ offspring in each iteration. Recall that each offspring created here can be considered as one created by standard bit mutation with mutation rate $pc$. We apply Lemma~\ref{lem:mutation-res} to one particular offspring $y'$ with $r = pcn$ and $\Delta = \frac{r}{4} + \delta$ and obtain
    \begin{align*}
        \Pr\left[f(y') - f(x) \ge \delta\right] &= \Pr\left[f(y') - f(x) \ge (n - 2 f(x)) \frac{r}{n} + (2f(x) - n) \frac{r}{n} + \delta\right] \\
        &\le \Pr\left[f(y') - f(x) \ge (n - 2 f(x)) \frac{r}{n} + \frac{r}{4} + \delta\right] \\
        &\le \exp\left(-\frac{\left(\frac{r}{4} + \delta\right)^2}{2\left(1 - \frac{r}{n}\right)\left(r + \frac{r}{12} + \frac{\delta}{3}\right)}\right) \\
        &\le \exp\left(-\frac{\left(\frac{r}{4} + \delta\right)^2}{2\left(\frac{13r}{12} + \frac{\delta}{3}\right)}\right).
    \end{align*}  
    For all $\delta > 0$ and $r > 0$ we bound the argument of the exponent as follows.
    \begin{align*}
        \frac{\left(\frac{r}{4} + \delta\right)^2}{2\left(\frac{13r}{12} + \frac{\delta}{3}\right)} &= \frac{3}{2} \cdot \frac{\left(\frac{13r}{4} + \delta - 3r\right)\left(\frac{r}{4} + \delta\right)}{\left(\frac{13r}{4} + \delta\right)} = \frac{3}{2} \left(1 - \frac{3r}{\frac{13r}{4} + \delta}\right) \left(\frac{r}{4} + \delta\right) \\
        &\ge \frac{3}{2} \left(1 - \frac{12}{13}\right) \delta = \frac{3}{26}\delta.
    \end{align*}
    Recall that $\delta = \frac{26}{3}(k + 2) \ln(\frac{n}{k})$. Hence, we have 
    \begin{align*}
        \Pr\left[f(y') - f(x) \ge \delta\right] &\le \exp\left(-\frac{\left(\frac{r}{4} + \delta\right)^2}{2\left(\frac{13r}{12} + \frac{\delta}{3}\right)}\right) \le e^{-\frac{3\delta}{26}} \\
        &= \exp\left(-(k + 2) \ln\left(\frac{n}{k}\right)\right) = \left(\frac{n}{k}\right)^{-(k + 2)}.
    \end{align*} 

    By the union bound the probability that we create such offspring in $\frac{n/4 - k}{\delta}$ iterations is at most 
    \begin{align*}
        \frac{\frac{n}{4} - k}{\delta} \lambda_M\lambda_C \left(\frac{n}{k}\right)^{-(k + 2)} \le \frac{n}{4\delta} \cdot \frac{1}{\ln\left(\frac{n}{k}\right)}\left(\frac{n}{k}\right)^{k + 1} \left(\frac{n}{k}\right)^{-(k + 2)} \le \frac{k}{4\delta} = \frac{3k}{26(k + 2)\ln\left(\frac{n}{k}\right)} \le \frac{3}{26}.
    \end{align*}

    If we start at some point $x$ with fitness $f(x) \le \frac{3n}{4}$ and we do not improve fitness by at least $\delta$ for $\frac{n/4 - k}{\delta}$ iterations, then we do not reach the local optima or the global optimum in this number of iterations (note that for the considered values of $k \le \frac{n}{32}$ and $\delta$ the value of $\frac{n/4 - k}{\delta}$ is at least one). During these iterations we do at least $(\lambda_M + \lambda_C)\frac{n/4 - k}{\delta}$ fitness evaluations. Since we have already shown that $\lambda_M + \lambda_C \ge 2\sqrt{\ln(\frac{n}{k})}(\frac{n}{k})^{\frac{k - 1}{2}}$, this is at least
    \begin{align*}
        2\sqrt{\ln\left(\frac{n}{k}\right)}\left(\frac{n}{k}\right)^{\frac{k - 1}{2}} \cdot \frac{\frac{n}{4} - k}{\delta} &= \left(\frac{n}{k}\right)^{\frac{k + 1}{2}} \cdot \frac{(n - 4k)k}{2\delta n}\sqrt{\ln\left(\frac{n}{k}\right)}
    \end{align*}
    fitness evaluations. We now estimate the factor $\frac{(n - 4k)k}{2\delta}\sqrt{\ln(\frac{n}{k})}$. Since by the theorem conditions we have $k \le \frac{n}{32}$, for $n$ large enough we have
    \begin{align*}
        \frac{(n - 4k)k}{2\delta n}\sqrt{\ln\left(\frac{n}{k}\right)} \ge \frac{7nk}{16\delta n} \sqrt{\ln\left(\frac{n}{k}\right)}= \frac{21k\sqrt{\ln\left(\frac{n}{k}\right)}}{416(k + 2) \ln\left(\frac{n}{k}\right)} \ge \frac{k}{20 \cdot 2k \sqrt{\ln\left(\frac{n}{k}\right)}} \ge \frac{1}{40\sqrt{\ln\left(\frac{n}{k}\right)}}.
    \end{align*}

    \textbf{Summary of the proof.} We now bring our arguments together. For the narrowed range of parameters we have shown that (i) with probability $\frac{1}{2} - e^{-\Theta(n)}$ the initial individual has between $\frac{n}{2}$ and $\frac{5n}{8}$ one-bits, (ii) then with probability $1 - e^{-\Theta(n)}$ we reach a point which has between $\frac{5n}{8}$ and $\frac{3n}{4}$ one-bits or exceed $\left(\frac{n}{k}\right)^{\frac{k + 1}{2}}$ fitness evaluations, (iii) then with probability at least $1 - \frac{3}{26} = \frac{23}{26}$ we do not reach the local optima or the global optimum in  $\frac{1}{40\sqrt{\ln(n/k)}}\left(\frac{n}{k}\right)^{\frac{k + 1}{2}}$ fitness evaluations. Hence, with probability at least 
    \begin{align*}
        \left(\frac{1}{2} - e^{-\Theta(n)}\right)\left(1 - e^{-\Theta(n)}\right) \frac{23}{26} = \frac{23}{52} - e^{-\Theta(n)}
    \end{align*}
    we do not find the optimum before making $\frac{1}{40\sqrt{\ln(n/k)}}\left(\frac{n}{k}\right)^{\frac{k + 1}{2}}$ fitness evaluations. Therefore, the expected runtime $T_F$ (in terms of fitness evaluations) is at least
    \begin{align*}
        E[T_F] \ge \left(\frac{23}{52} - e^{-\Theta(n)}\right)\frac{1}{40\sqrt{\ln\left(\frac{n}{k}\right)}}\left(\frac{n}{k}\right)^{\frac{k + 1}{2}} \ge \frac{1}{91\sqrt{\ln\left(\frac{n}{k}\right)}}\left(\frac{n}{k}\right)^{\frac{k + 1}{2}}.
    \end{align*}

\end{proof}

\section{Experiments}
\label{sec:experiments}

As our theoretical analysis gives upper bounds that are precise only up to constant factors, we now use experiments to obtain a better understanding of how the heavy-tailed \ollga performs on concrete problem sizes. We conducted a series of experiments on \onemax and \jump functions with jump sizes $k \in [2..6]$.

Since our theory-based recommendations for $\beta_p$ and $\beta_c$ are very similar, the analysis on the jump functions treats the corresponding distributions very symmetrically, and our preliminary experimentation did not find any significant advantages from using different values for $\beta_p$ and $\beta_c$, we decided to keep them equal in our experiments and denote them together as $\beta_{pc}$, such that $\beta_p = \beta_c = \beta_{pc}$.

In all the presented plots we display average values of 100 independent runs of the considered algorithm, together with the standard deviation.

\subsection{Results for \onemax}
\label{sec:experiments-onemax}

For \onemax we considered the problem sizes $n \in \{ 2^i \mid 3 \le i \le 14 \}$ and all combinations of choices of $\beta_{pc} \in \{ 1.0, 1.2, 1.4, 1.6, 1.8, 2.0, 2.2 \}$ and $\beta_\lambda \in \{ 2.0, 2.2, 2.4, 2.6, 2.8, 3.0, 3.2 \}$. For reasons of space, we only present a selection of these results.

\begin{figure}[!t]
\begin{tikzpicture}
    \begin{loglogaxis}[width=\linewidth, height=0.3\textheight,
                       legend pos=north west, legend columns=2, ymin=0.7,
                       cycle list name=myplotcycle, grid=major, log base x=2, enlarge x limits=0.05,
                       enlarge y limits=false, ymin=1, ymax=100,
                       xlabel={Problem size $n$}, ylabel={Runtime / $n \ln (n)$}]
                       \addplot plot [error bars/.cd, y dir=both, y explicit] coordinates {(8,4.344)+-(0,3.385)(16,5.578)+-(0,3.191)(32,6.254)+-(0,3.101)(64,7.437)+-(0,2.732)(128,8.828)+-(0,2.899)(256,11.84)+-(0,2.813)(512,13.57)+-(0,2.948)(1024,17.33)+-(0,3.011)(2048,20.65)+-(0,3.15)(4096,23.58)+-(0,2.748)(8192,27.68)+-(0,2.835)(16384,32.06)+-(0,2.305)};
        \addlegendentry{$\beta_{pc}=1.0$};
        \addplot plot [error bars/.cd, y dir=both, y explicit] coordinates {(8,4.286)+-(0,3.372)(16,4.928)+-(0,2.564)(32,4.919)+-(0,1.836)(64,6.723)+-(0,2.26)(128,7.48)+-(0,2.141)(256,8.927)+-(0,2.253)(512,9.857)+-(0,2.093)(1024,11.4)+-(0,1.941)(2048,12.39)+-(0,1.575)(4096,14.07)+-(0,1.696)(8192,15.51)+-(0,1.399)(16384,16.51)+-(0,1.205)};
        \addlegendentry{$\beta_{pc}=1.2$};
        \addplot plot [error bars/.cd, y dir=both, y explicit] coordinates {(8,3.941)+-(0,2.315)(16,4.925)+-(0,2.767)(32,4.894)+-(0,2.178)(64,5.676)+-(0,2.002)(128,5.861)+-(0,1.444)(256,6.974)+-(0,1.546)(512,7.603)+-(0,1.504)(1024,8.274)+-(0,1.316)(2048,8.652)+-(0,1.161)(4096,9.251)+-(0,0.8843)(8192,9.805)+-(0,0.8419)(16384,10.16)+-(0,0.6273)};
        \addlegendentry{$\beta_{pc}=1.4$};
        \addplot plot [error bars/.cd, y dir=both, y explicit] coordinates {(8,4.019)+-(0,2.227)(16,3.998)+-(0,1.91)(32,4.418)+-(0,1.883)(64,5.121)+-(0,2.175)(128,5.372)+-(0,1.673)(256,5.608)+-(0,1.306)(512,5.941)+-(0,1.144)(1024,6.471)+-(0,0.8834)(2048,6.634)+-(0,0.8332)(4096,6.82)+-(0,0.7024)(8192,7.003)+-(0,0.6147)(16384,7.202)+-(0,0.4673)};
        \addlegendentry{$\beta_{pc}=1.6$};
        \addplot plot [error bars/.cd, y dir=both, y explicit] coordinates {(8,4.365)+-(0,2.655)(16,4.163)+-(0,2.446)(32,4.376)+-(0,1.827)(64,4.89)+-(0,1.909)(128,4.722)+-(0,1.281)(256,5.009)+-(0,0.985)(512,5.173)+-(0,1.004)(1024,5.272)+-(0,0.8602)(2048,5.417)+-(0,0.6877)(4096,5.482)+-(0,0.5393)(8192,5.505)+-(0,0.5136)(16384,5.564)+-(0,0.3547)};
        \addlegendentry{$\beta_{pc}=1.8$};
        \addplot plot [error bars/.cd, y dir=both, y explicit] coordinates {(8,3.889)+-(0,2.795)(16,4.26)+-(0,3.202)(32,3.847)+-(0,1.604)(64,4.383)+-(0,1.41)(128,4.086)+-(0,1.046)(256,4.356)+-(0,0.919)(512,4.48)+-(0,0.8339)(1024,4.534)+-(0,0.6551)(2048,4.597)+-(0,0.5514)(4096,4.612)+-(0,0.4825)(8192,4.678)+-(0,0.362)(16384,4.732)+-(0,0.3363)};
        \addlegendentry{$\beta_{pc}=2.0$};
        \addplot plot [error bars/.cd, y dir=both, y explicit] coordinates {(8,3.89)+-(0,2.276)(16,3.647)+-(0,1.856)(32,3.827)+-(0,1.383)(64,4.031)+-(0,1.401)(128,4.227)+-(0,1.188)(256,4.095)+-(0,0.8409)(512,4.12)+-(0,0.7246)(1024,4.052)+-(0,0.5681)(2048,4.155)+-(0,0.5134)(4096,4.095)+-(0,0.3546)(8192,4.14)+-(0,0.3234)(16384,4.099)+-(0,0.256)};
        \addlegendentry{$\beta_{pc}=2.2$};
        \addplot plot [error bars/.cd, y dir=both, y explicit] coordinates {(8,1.922)+-(0,1.228)(16,2.203)+-(0,1.156)(32,2.175)+-(0,0.7782)(64,2.323)+-(0,0.7722)(128,2.315)+-(0,0.7005)(256,2.413)+-(0,0.5527)(512,2.468)+-(0,0.6407)(1024,2.356)+-(0,0.4572)(2048,2.423)+-(0,0.4443)(4096,2.435)+-(0,0.3902)(8192,2.515)+-(0,0.4193)(16384,2.474)+-(0,0.3908)};
        \addlegendentry{(1+1) EA};
    \end{loglogaxis}
\end{tikzpicture}
\caption{Running times of the heavy-tailed \ollga on \onemax
         starting from a random point,
         normalized by $n \ln (n)$,
         for different $\beta_{pc} = \beta_p = \beta_c$ and $\beta_\lambda=2.8$
         in relation to the problem size $n$.
         The expected running times of \oea, also starting from a random point, are given for comparison.}
\label{plot:om:pc}
\end{figure}
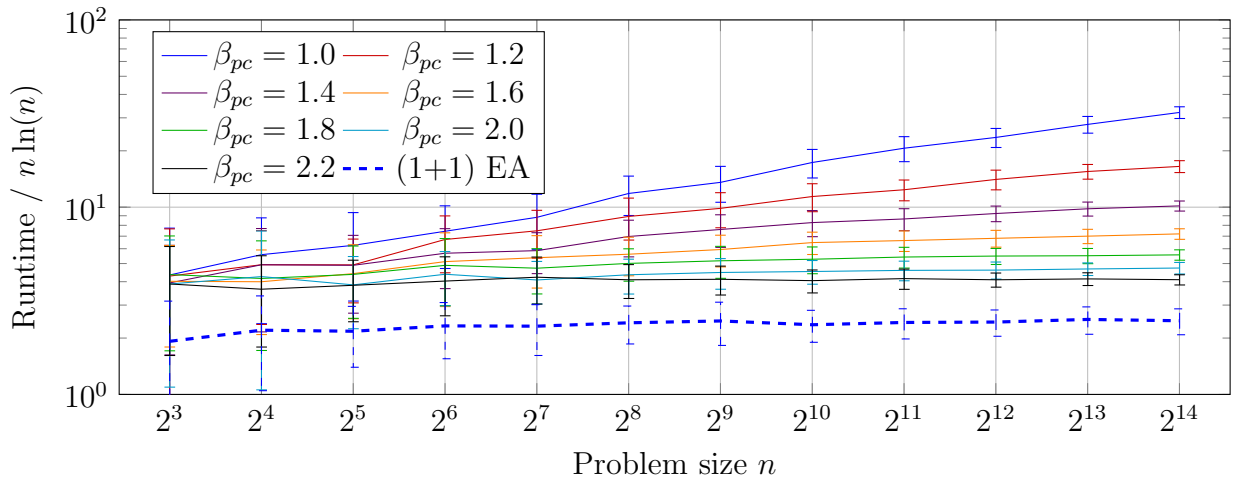

Figure~\ref{plot:om:pc} presents the running times of the heavy-tailed \ollga on \onemax against the problem size $n$ for all considered values of $\beta_{pc}$, whereas a fixed value of $\beta_\lambda = 2.8$ is used. The running times of the \oea are also presented for comparison. Since the runtime is normalized by $n \ln (n)$, the plot of the latter tends to a horizontal line, and so do the plots of the heavy-tailed \ollga with $\beta_{pc} \ge 1.8$. Other plots, after discounting for the noise in the measurements, appear to be convex upwards and, similarly to~\cite{AntipovBD22}, they will likely become horizontal as $n$ grows. For \onemax, bigger values of $\beta_{pc}$ appear to be better. Since greater $\beta_{pc}$ increases the chances of behaving similar to the \oea during an iteration, this fits to the situation discussed in  Lemma~\ref{lem:opo_iteration}.
The plots look similar also for $\beta_\lambda$ different from $2.8$, so we do not present them here.

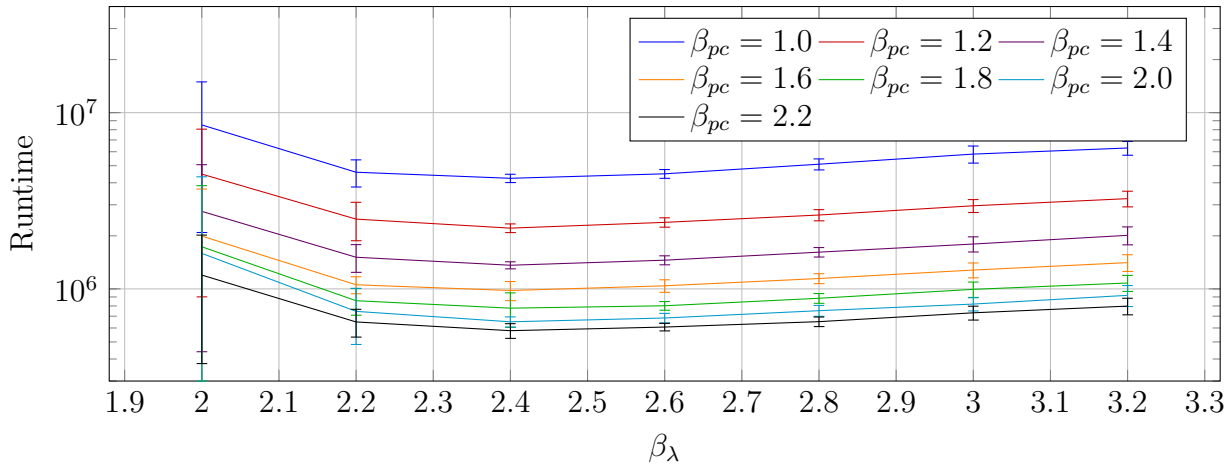
\begin{figure}[!t]
\begin{tikzpicture}
    \begin{semilogyaxis}[width=\linewidth, height=0.3\textheight,
                 legend pos=north east, legend columns=3,
                 cycle list name=myplotcycle, grid=major,
                 xlabel={$\beta_\lambda$}, ylabel={Runtime}, ymin=300000, ymax=40000000]
        \addplot plot [error bars/.cd, y dir=both, y explicit] coordinates {(2.0,8.517e+06)+-(0,6.428e+06)(2.2,4.591e+06)+-(0,8.056e+05)(2.4,4.241e+06)+-(0,2.301e+05)(2.6,4.497e+06)+-(0,2.571e+05)(2.8,5.098e+06)+-(0,3.664e+05)(3.0,5.819e+06)+-(0,6.471e+05)(3.2,6.301e+06)+-(0,5.688e+05)};
        \addlegendentry{$\beta_{pc}=1.0$};
        \addplot plot [error bars/.cd, y dir=both, y explicit] coordinates {(2.0,4.48e+06)+-(0,3.579e+06)(2.2,2.488e+06)+-(0,6.119e+05)(2.4,2.212e+06)+-(0,1.252e+05)(2.6,2.385e+06)+-(0,1.455e+05)(2.8,2.625e+06)+-(0,1.915e+05)(3.0,2.962e+06)+-(0,2.495e+05)(3.2,3.251e+06)+-(0,3.312e+05)};
        \addlegendentry{$\beta_{pc}=1.2$};
        \addplot plot [error bars/.cd, y dir=both, y explicit] coordinates {(2.0,2.753e+06)+-(0,2.313e+06)(2.2,1.512e+06)+-(0,2.708e+05)(2.4,1.362e+06)+-(0,6.273e+04)(2.6,1.455e+06)+-(0,8.552e+04)(2.8,1.616e+06)+-(0,9.973e+04)(3.0,1.797e+06)+-(0,1.776e+05)(3.2,2.013e+06)+-(0,2.348e+05)};
        \addlegendentry{$\beta_{pc}=1.4$};
        \addplot plot [error bars/.cd, y dir=both, y explicit] coordinates {(2.0,1.994e+06)+-(0,1.692e+06)(2.2,1.055e+06)+-(0,1.169e+05)(2.4,9.791e+05)+-(0,1.215e+05)(2.6,1.041e+06)+-(0,8.516e+04)(2.8,1.145e+06)+-(0,7.43e+04)(3.0,1.279e+06)+-(0,1.242e+05)(3.2,1.409e+06)+-(0,1.531e+05)};
        \addlegendentry{$\beta_{pc}=1.6$};
        \addplot plot [error bars/.cd, y dir=both, y explicit] coordinates {(2.0,1733579.7)+=(0,2118321.2)-=(0,1433579.7)(2.2,8.568e+05)+-(0,1.479e+05)(2.4,7.774e+05)+-(0,1.721e+05)(2.6,8.021e+05)+-(0,4.554e+04)(2.8,8.846e+05)+-(0,5.64e+04)(3.0,9.94e+05)+-(0,9.945e+04)(3.2,1.08e+06)+-(0,1.128e+05)};
        \addlegendentry{$\beta_{pc}=1.8$};
        \addplot plot [error bars/.cd, y dir=both, y explicit] coordinates {(2.0,1590446)+=(0,2735310.5)-=(0,1290446)(2.2,7.456e+05)+-(0,2.618e+05)(2.4,6.505e+05)+-(0,4.32e+04)(2.6,6.837e+05)+-(0,4.243e+04)(2.8,7.523e+05)+-(0,5.347e+04)(3.0,8.201e+05)+-(0,7.058e+04)(3.2,9.192e+05)+-(0,1.253e+05)};
        \addlegendentry{$\beta_{pc}=2.0$};
        \addplot plot [error bars/.cd, y dir=both, y explicit] coordinates {(2.0,1.198e+06)+-(0,8.211e+05)(2.2,6.493e+05)+-(0,1.166e+05)(2.4,5.793e+05)+-(0,5.555e+04)(2.6,6.074e+05)+-(0,2.996e+04)(2.8,6.518e+05)+-(0,4.071e+04)(3.0,7.316e+05)+-(0,6.618e+04)(3.2,7.987e+05)+-(0,8.601e+04)};
        \addlegendentry{$\beta_{pc}=2.2$};                 
    \end{semilogyaxis}
\end{tikzpicture}
\caption{Running times of the heavy-tailed \ollga on \onemax
         starting from a random point,
         for $n=2^{14}$ and different $\beta_{pc} = \beta_p = \beta_c$
         depending on $\beta_\lambda$.}
\label{plot:om:lambda}
\end{figure}

To investigate the dependencies on $\beta_\lambda$ and $\beta_{pc}$ more thoroughly, we consider the largest available problem size $n=2^{14}$ and plot the runtimes for all parameters configurations in Figure~\ref{plot:om:lambda}. The general trend of an improving  performance with growing $\beta_{pc}$ can be clearly seen here as well. For $\beta_\lambda$, the picture is less clear. It appears that very small $\beta_\lambda$ also result in larger running times, medium values of roughly $\beta_\lambda = 2.4$ yield the best available runtimes, and a further increase of $\beta_\lambda$ increases the runtime again, but only slightly. As very large $\beta_\lambda$, such as $\beta_\lambda = 3.2$, correspond to regimes similar to the \oea, this might be a sign that some of the working principles of the \ollga are still beneficial on an easy problem like \onemax.

\subsection{Results for \jump Functions}

For \jump functions we used the problem sizes $n \in \{ 2^i \mid 3 \le i \le 7\}$, subject to the condition $k \le \frac{n}{4}$ and hence $n \ge 4k$, as assumed in the theoretical results of this paper. As running times are higher in this setting, 
we consider a smaller set of parameter combinations, $\beta_{pc} \in \{ 1.0, 1.2, 1.4 \}$ and $\beta_\lambda \in \{ 2.0, 2.2, 2.4 \}$.

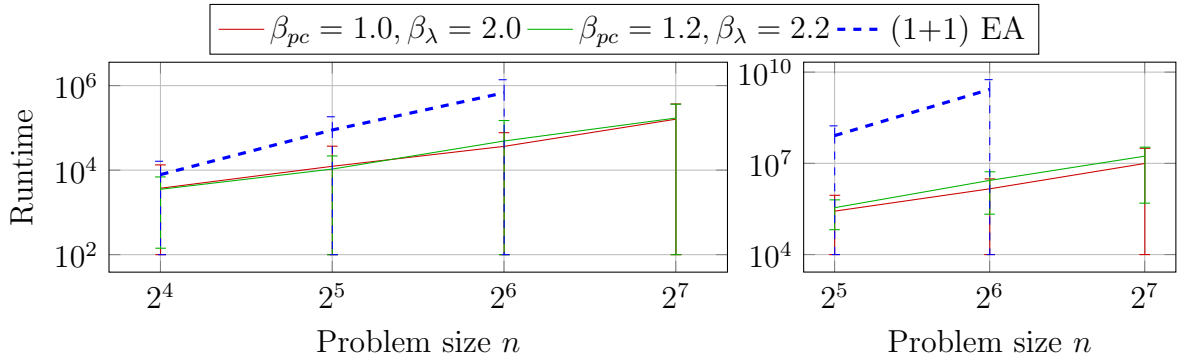
\begin{figure}[!t]
\begin{tikzpicture}
    \begin{loglogaxis}[name=ax1,
                       width=0.6\linewidth, height=0.2\textheight,
                       legend style={at={(0.8333,1)}, anchor=south,legend columns=3, yshift=0.1cm},
                       cycle list name=twoconfigsandea, grid=major, log base x=2,
                       xlabel={Problem size $n$}, ylabel={Runtime}]
        \addplot plot [error bars/.cd, y dir=both, y explicit] coordinates {(16,3726.5)+=(0,9569.256)-=(0,3626.5)(32,12333.34)+=(0,24515.504)-=(0,12233.34)(64,36516.26)+=(0,40366.678)-=(0,36416.26)(128,160458.46)+=(0,206772.6)-=(0,160358.46)};
        \addlegendentry{$\beta_{pc}=1.0, \beta_\lambda=2.0$};
        \addplot plot [error bars/.cd, y dir=both, y explicit] coordinates {(16,3522)+-(0,3380)(32,10595.18)+=(0,11075.924)-=(0,10495.18)(64,48900.78)+=(0,100280.37)-=(0,48800.78)(128,170846.72)+=(0,192408.3)-=(0,170746.72)};
        \addlegendentry{$\beta_{pc}=1.2, \beta_\lambda=2.2$};
        \addplot plot [error bars/.cd, y dir=both, y explicit] coordinates {(16,7832.27)+=(0,8376.3934)-=(0,7732.27)(32,89135.28)+=(0,93856.833)-=(0,89035.28)(64,675625.48)+=(0,700737.88)-=(0,675525.48)};
        \addlegendentry{(1+1) EA};
    \end{loglogaxis}
    \begin{loglogaxis}[at={ (ax1.south east) }, xshift=1cm, width=0.4\linewidth, height=0.2\textheight,
                       every axis legend/.code={\renewcommand\addlegendentry[2][]{}},
                       cycle list name=twoconfigsandea, grid=major, log base x=2,
                       xlabel={Problem size $n$}]
        \addplot plot [error bars/.cd, y dir=both, y explicit] coordinates {(32,267973.94)+=(0,622973.04)-=(0,257973.94)(64,1453299.6)+=(0,1627890.1)-=(0,1443299.6)(128,9912644.4)+=(0,21309415)-=(0,9902644.4)};
        \addlegendentry{$\beta_{pc}=1.0, \beta_\lambda=2.0$};
        \addplot plot [error bars/.cd, y dir=both, y explicit] coordinates {(32,3.482e+05)+-(0,2.825e+05)(64,2.749e+06)+-(0,2.536e+06)(128,1.724e+07)+-(0,1.675e+07)};
        \addlegendentry{$\beta_{pc}=1.2, \beta_\lambda=2.2$};
        \addplot plot [error bars/.cd, y dir=both, y explicit] coordinates {(32,81719198)+=(0,89274148)-=(0,81709198)(64,2.7020314e+09)+=(0,2.9516463e+09)-=(0,2.7020214e+09)};
        \addlegendentry{(1+1) EA};
    \end{loglogaxis}
\end{tikzpicture}
\caption{Running times of the heavy-tailed \ollga on \jump, depending on the problem size $n$, in comparison to the \oea.
         Jump sizes are $k=3$ on the left and $k=5$ on the right. }
\label{plot:jump:comparison}
\end{figure}

Figure~\ref{plot:jump:comparison} presents the results of a comparison of the heavy-tailed \ollga with the \oea on \jump with jump parameter $k \in \{3, 5\}$. We chose two most distant distribution parameters for the heavy-tailed \ollga for presenting in this figure. However, the difference between these is negligible compared to the difference to the \oea. Such a difference aligns well with the theory, as the running time of the \oea is $\Theta(n^k)$, whereas Theorem~\ref{thm:jump} predicts much smaller running times for the heavy-tailed \ollga.

\begin{table}[!t]
\caption{The $p$-values for experimental results presented in Figure~\ref{plot:jump:comparison}}\label{stats:jump:comparison}
\centering
\begin{tabular}{llll}\hline
$k$ & $n$ & Student's t-test      & Wilcoxon rank sum test \\\hline
3   & 16  & $1.46 \cdot 10^{-3}$  & $4.14 \cdot 10^{-6}$ \\
3   & 32  & $1.88 \cdot 10^{-12}$ & $2.87 \cdot 10^{-20}$ \\
3   & 64  & $2.59 \cdot 10^{-14}$ & $5.29 \cdot 10^{-26}$ \\
5   & 32  & $9.32 \cdot 10^{-15}$ & $1.58 \cdot 10^{-34}$ \\
5   & 64  & $8.01 \cdot 10^{-15}$ & $1.58 \cdot 10^{-34}$ \\\hline
\end{tabular}
\end{table}

Due to large standard deviation, we performed statistical tests on the results presented in Figure~\ref{plot:jump:comparison} using two statistical tests: the Student's t-test as the one which checks mean values which are the subject of our theorems, and the Wilcoxon rank sum test as a non-parametric test. The results are presented in Table~\ref{stats:jump:comparison}, where for each row and each test the maximum $p$-value is shown out of two between the \oea and either of the parameterizations of the heavy-tailed \ollga. The $p$-values in all the cases are very small: except for the case $k=3, n=16$, they are all well below $10^{-10}$, which indicates a vast difference between the algorithms and hence a clear superiority of the heavy-tailed \ollga.

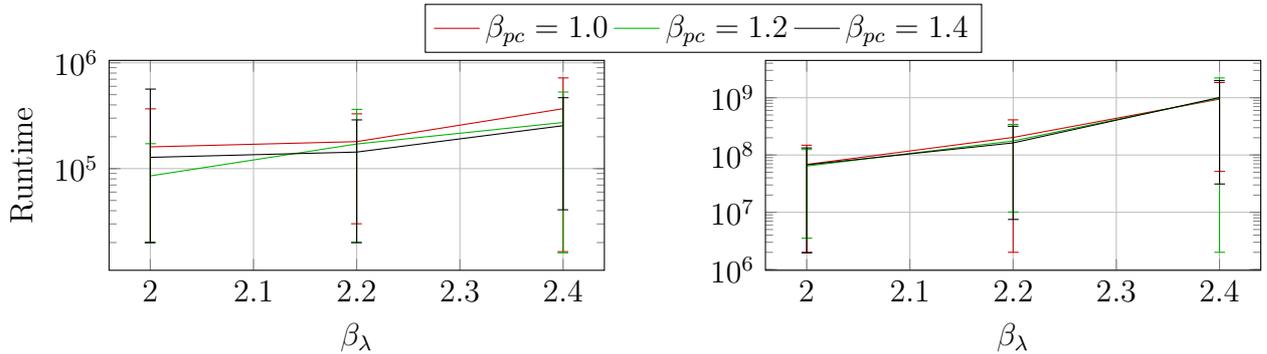
\begin{figure}[!t]
\begin{tikzpicture}
    \begin{semilogyaxis}[name=ax2, width=0.5\linewidth, height=0.2\textheight,
                       legend style={at={(1.2,1)}, anchor=south,legend columns=3, yshift=0.1cm},
                       cycle list name=threeconfigs, grid=major,
                       xlabel={$\beta_\lambda$}, ylabel={Runtime}]
        \addplot plot [error bars/.cd, y dir=both, y explicit] coordinates {(2.0,160458.46)+=(0,206772.6)-=(0,140458.46)(2.2,1.799e+05)+-(0,1.498e+05)(2.4,3.684e+05)+-(0,3.52e+05)};
        \addlegendentry{$\beta_{pc}=1.0$};
        \addplot plot [error bars/.cd, y dir=both, y explicit] coordinates {(2.0,85332.94)+=(0,86733.084)-=(0,65332.94)(2.2,170846.72)+=(0,192408.3)-=(0,150846.72)(2.4,2.729e+05)+-(0,2.569e+05)};
        \addlegendentry{$\beta_{pc}=1.2$};
        \addplot plot [error bars/.cd, y dir=both, y explicit] coordinates {(2.0,127732)+=(0,436867.86)-=(0,107732)(2.2,143298.94)+=(0,145536.47)-=(0,123298.94)(2.4,2.545e+05)+-(0,2.137e+05)};
        \addlegendentry{$\beta_{pc}=1.4$};
    \end{semilogyaxis}
    \begin{semilogyaxis}[at={ (ax1.south east) }, xshift=0.5cm, width=0.5\linewidth, height=0.2\textheight,
                       every axis legend/.code={\renewcommand\addlegendentry[2][]{}},
                       cycle list name=threeconfigs, grid=major,
                       xlabel={$\beta_\lambda$}]
        \addplot plot [error bars/.cd, y dir=both, y explicit] coordinates {(2.0,68367919)+=(0,79905811)-=(0,66367919)(2.2,2.0260088e+08)+=(0,2.0864976e+08)-=(0,2.0060088e+08)(2.4,9.465e+08)+-(0,8.948e+08)};
        \addlegendentry{$\beta_{pc}=1.0$};
        \addplot plot [error bars/.cd, y dir=both, y explicit] coordinates {(2.0,6.412e+07)+-(0,6.058e+07)(2.2,1.76e+08)+-(0,1.659e+08)(2.4,9.8176449e+08)+=(0,1.2349597e+09)-=(0,9.7976449e+08)};
        \addlegendentry{$\beta_{pc}=1.2$};
        \addplot plot [error bars/.cd, y dir=both, y explicit] coordinates {(2.0,6.737e+07)+-(0,6.541e+07)(2.2,1.628e+08)+-(0,1.553e+08)(2.4,1.011e+09)+-(0,9.798e+08)};
        \addlegendentry{$\beta_{pc}=1.4$};
    \end{semilogyaxis}
\end{tikzpicture}
\caption{Dependency of running times of the heavy-tailed \ollga
         on $\jump_k$ on $\beta_\lambda$ and $\beta_{pc}$
         for $k=3$ (on the left) and $k=6$ (on the right). Problem size $n=2^7$ is used.}
\label{plot:jump:lambda}
\end{figure}

The parameter study, presented in Figure~\ref{plot:jump:lambda}, suggests that for \jump the particular values of $\beta_{pc}$ are not very important, although larger values result in the marginally better performance. However, larger $\beta_\lambda$ tend to make the performance worse, which is more pronounced for larger jump sizes $k$. This finding agrees with upper bound proven in Corollary~\ref{cor:hyperparameters}, in which $(n/k)$ is raised to a power that is proportional to $\eps_\lambda = \beta_\lambda-1$.
Each difference is statistically significant with $p < 0.008$ using the Wilcoxon rank sum test.

\begin{figure}[!t]
\begin{tikzpicture}
    \begin{loglogaxis}[width=\linewidth, height=0.23\textheight,
                       legend pos=north west,
                       cycle list name=myplotcycle, grid=major, log base x=2,
                       xlabel={Problem size $n$}, ylabel={Runtime}]
        \addplot plot [error bars/.cd, y dir=both, y explicit] coordinates {(32,1076964.9)+=(0,1199914.1)-=(0,1076944.9)(64,8611206)+=(0,13040599)-=(0,8611186)(128,68367919)+=(0,79905811)-=(0,68367899)};
        \addlegendentry{$k=6$};
        \addplot plot [error bars/.cd, y dir=both, y explicit] coordinates {(32,267973.94)+=(0,622973.04)-=(0,267953.94)(64,1453299.6)+=(0,1627890.1)-=(0,1453279.6)(128,9912644.4)+=(0,21309415)-=(0,9912624.4)};
        \addlegendentry{$k=5$};
        \addplot plot [error bars/.cd, y dir=both, y explicit] coordinates {(16,16129.54)+=(0,19970.327)-=(0,16109.54)(32,52294.56)+=(0,53395.074)-=(0,52274.56)(64,1.799e+05)+-(0,1.635e+05)(128,1023150.8)+=(0,2015682.9)-=(0,1023130.8)};
        \addlegendentry{$k=4$};
        \addplot plot [error bars/.cd, y dir=both, y explicit] coordinates {(16,3726.5)+=(0,9569.256)-=(0,3706.5)(32,12333.34)+=(0,24515.504)-=(0,12313.34)(64,36516.26)+=(0,40366.678)-=(0,36496.26)(128,160458.46)+=(0,206772.6)-=(0,160438.46)};
        \addlegendentry{$k=3$};
        \addplot plot [error bars/.cd, y dir=both, y explicit] coordinates {(8,201.3)+-(0,174)(16,875.96)+=(0,883.30886)-=(0,855.96)(32,2345.84)+=(0,3612.3183)-=(0,2325.84)(64,5981)+-(0,4922)(128,24702.14)+=(0,36252.83)-=(0,24682.14)};
        \addlegendentry{$k=2$};                       
    \end{loglogaxis}
\end{tikzpicture}
\caption{Running times of the heavy-tailed \ollga with $\beta_{pc} = 1.0$ and $\beta_\lambda=2.0$
         on $\jump_k$ for $k\in[2..6]$
         depending on the problem size $n$.}
\label{plot:jump:all-k}
\end{figure}
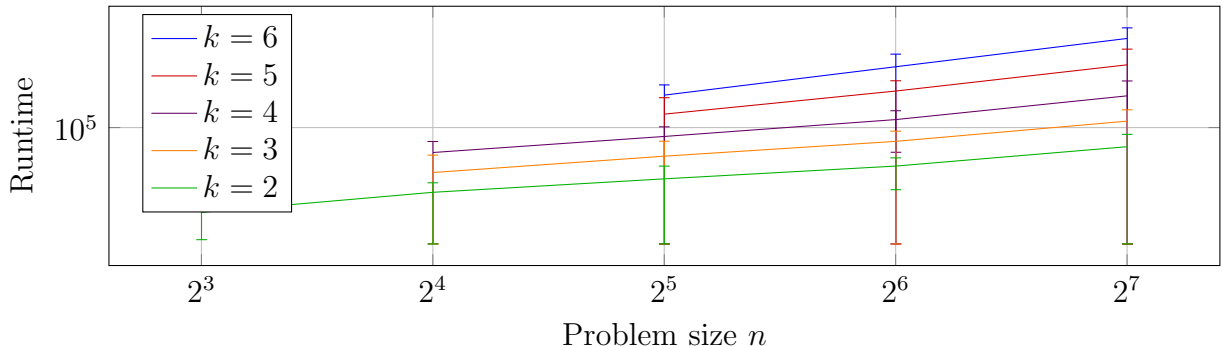

Finally, Figure~\ref{plot:jump:all-k} shows the running times of the heavy-tailed \ollga for a fixed parameterization $\beta_{pc} = 1.0$ and $\beta_\lambda = 2.0$ for all available values of $n$ and $k$, to give an impression of the typical running times of this algorithm on the \jump problem.

\section{Conclusion}

Using mathematical and experimental methods, we showed that choosing all parameters of an algorithm randomly from a power-law distribution can lead to a very good performance both on easy unimodal and on multimodal problems. This lazy approach to the parameter tuning and control problem requires very little understanding how the parameters influence the algorithm behavior. The only design choice left to the algorithm user is deciding the scaling of the parameters, but we observed that the natural choices worked out very well. Our empirical and theoretical studies show that the precise choice of the (other)  parameters of the power-law distributions does not play a significant role and they both suggest to use unbounded power-law distributions and to take a power-law exponent $2 + \eps$ for the population size (so that the expected cost of one iteration is constant) and $1 + \eps$ for other parameters (to maximize the positive effect of a heavy-tailed distribution). With these considerations, one may call our approach essentially \emph{parameter-less}. 

Surprisingly, our randomized parameter choice even yields a runtime which is better than the best proven runtime for optimal static parameters on \jump functions of some jump sizes. An interesting question (which we leave open for the further research) is whether a random parameter choice can outperform the algorithms with known optimal static parameters also on other problems. The experiments on \onemax when starting with a good solution (in distance $\sqrt{n}$ from the optimum) indicate that it is possible since there we observed a benefit from the key mechanisms of the \ollga.

\section*{Acknowledgements}
    This work was supported by RFBR and CNRS, project number 20-51-15009 and by a public grant as part of the Investissements d'avenir project, reference ANR-11-LABX-0056-LMH, LabEx LMH.

\section*{Declarations}
The authors have no relevant financial or non-financial interests to disclose.


\newcommand{\etalchar}[1]{$^{#1}$}

}

\end{document}